\documentclass{article}
\pdfoutput=1

\usepackage{PRIMEarxiv}

\usepackage[utf8]{inputenc} % allow utf-8 input
\usepackage[T1]{fontenc}    % use 8-bit T1 fonts
\usepackage{hyperref}       % hyperlinks
\usepackage{url}            % simple URL typesetting
\usepackage{booktabs}       % professional-quality tables
\usepackage{amsfonts}       % blackboard math symbols
\usepackage{nicefrac}       % compact symbols for 1/2, etc.
\usepackage{microtype}      % microtypography
\usepackage{lipsum}
\usepackage{fancyhdr}       % header
\usepackage{graphicx}       % graphics

\usepackage{amsmath}
\usepackage{amssymb}
\usepackage{mathtools}
\usepackage{amsthm}

\newcommand{\argmin}{\operatornamewithlimits{argmin }}
\newcommand{\argmax}{\operatornamewithlimits{argmax }}

\usepackage{subcaption}

\usepackage[capitalize,noabbrev]{cleveref}

\usepackage{algorithmic}
\usepackage{algorithm}

\newtheorem{prop}{Proposition}

\graphicspath{{./Figures/}}
\usepackage{subcaption}
\usepackage{gensymb}

\usepackage{tikz}
\usetikzlibrary{shapes,arrows}

\title{Probabilistic Registration for  Gaussian Process 3D shape modelling in the presence of extensive missing data
	\thanks{This work has received funding from the European Union’s Horizon 2020 research and innovation programme  under the Marie Skłodowska-Curie Project BIGMATH, Grant Agreement No 812912.}
}

\author{
	Filipa Valdeira \\
	Department of Environmental Science and Policy \\
	Università Degli Studi di Milano \\
	Milan\\
	\texttt{filipa.marreiros@unimi.it} \\
	\And	 
	Ricardo Ferreira\\
	$\mu$Roboptics \\
	Lisbon \\
	\texttt{alessandra.micheletti@unimi.it} \\
	\And
	 Alessandra Micheletti\\
	Department of Environmental Science and Policy \\
	Università Degli Studi di Milano \\
	Milan\\
	\texttt{alessandra.micheletti@unimi.it} \\
	\And
	Cl\'audia Soares \\
	Computer Science Department, NOVA LINCS \\
	NOVA School of Science and Technology, Universidade NOVA de Lisboa \\
	Lisbon\\
	\texttt{claudia.soares@fct.unl.pt} \\
}

\graphicspath{{./Figures/}}

\begin{document}
	
	\maketitle
	
	\begin{abstract} We propose a shape fitting/registration method based on a Gaussian Processes formulation, suitable for shapes with extensive regions of missing data. Gaussian Processes are a proven powerful tool, as they provide a unified setting for shape modelling and fitting. While the existing methods in this area prove to work well for the general case of the human head, when looking at more detailed and deformed data, with a high prevalence of missing data, such as the ears, the results are not satisfactory. In order to overcome this, we formulate the shape fitting problem as a multi-annotator Gaussian Process Regression and establish a parallel with the standard probabilistic registration. The achieved method SFGP shows better performance when dealing with extensive areas of missing data when compared to a state-of-the-art registration method and current approaches for registration with pre-existing shape models. Experiments are conducted both for a 2D small dataset with diverse transformations and a 3D dataset of ears.
	\end{abstract}


	\section{Introduction}
	
	Consider the problem of predicting a complex shape, like a human ear, from a dataset of similar point clouds and partially observed points of this shape. In this setting, we face both the registration and modelling of different point clouds. There are increasing areas of application for 3D shape modelling, in particular when it comes to the human body, spread over medical applications (segmentation \cite{article:Apps_segmentation}, prosthesis design \cite{article:UnsupervisedPipeline}, surgical planning \cite{article:Apps_surgical_plan}), surveillance (face recognition \cite{article:3DMM_Basel_Face_Model}, tracking \cite{article:Apps_tracking}) or human-machine interaction (expression/emotion detection \cite{article:emotion}, virtual humans \cite{article:virtual_humans}). A large number of approaches focuses on the human head, with particular incidence on the face region, and increasingly better models have been achieved for coarse-grained requirements \cite{article:3DMM_large_scale,article:Gaussian_Process_Combining_3DMM}.

	However, small and detailed areas are still not accurately represented and attempts to tackle this problem are currently emerging \cite{article:Gaussian_Process_Combining_3DMM_Ear}. Our driving example is the modelling of an ear, representative of a fine-detailed region with extensive data problems. Given the challenging shape of this face segment, the 3D scan procedure leads to broad regions of missing data and a high level of noise. Our approach is nonetheless generic enough to be applied to any other shape, as seen in Figure~\ref{fig:Fish_Example} of Section~\ref{sec:experiments} with the 2D fish data.

	A standard approach to obtain a statistical shape model from a given dataset is through 3D Morphable Models (3DMM), first proposed in \cite{article:Morphable_Blanz_Vetter}. Although they include both a shape and an appearance component, here the focus is only on the former. Given a set of scans, the standard procedure to obtain a 3DMM entails two main steps: \textit{dense correspondence} and \textit{modelling}. During the former, the original samples of unorganized point clouds are set into correspondence, i.e., a re-parametrization is found such that points with the same index have the same semantic meaning (for instance, point $i$ of each scan represents the tip of the nose). This is a requirement for the subsequent \textit{modelling} step where the deformations of each sample are studied, in order to express shape variability in a lower dimension. Correspondence across a dataset can be achieved by deforming a generic reference shape to each target scan, without prior information on the shape characteristics. This approach is a subclass of correspondence methods denoted as \textit{registration} \cite{article:correspondence_Review} and the standard choice for shape modelling pipelines \cite{article:Morphable_Blanz_Vetter,article:3DMM_Basel_Face_Model}. Given that no shape prior is taken into account, registration is by itself a broad area of study, developed separately from the modelling setting. 
	
	Therefore, traditional approaches to 3DMM directly employ a state-of-the-art registration method, followed by Principal Component Analysis, to find a low-dimensional representation of shape variability \cite{article:Morphable_Blanz_Vetter,article:3DMM_Basel_Face_Model}. Further developments on the modelling side have included increasing the size of the training set \cite{article:3DMM_large_scale} for higher variability, using part-based models \cite{article:Model_partbased,article:Model_partbased_wavelets} instead of global ones or the use of different shape spaces \cite{article:sphere_space,article:tangentPCA}. Nonetheless, and particularly for shapes with non-rigid deformations, correspondence remains an open challenge and a limiting step on the quality of the resulting model \cite{article:Review_Non_Rigid_Shape_Analysis}. This has motivated proposals on dense correspondence specifically tailored for 3D faces in recent years \cite{article:Registration_Face_semantic_2018, article:Registration_Face_tangencial_2021,article:Registration_Face_Boosting_2019}.

	However, regardless of the registration method in place, an underlying model restricts deformations of the reference shape, such that unfeasible shapes are avoided. In fact, the process of registration can be understood, on a higher level of abstraction, as \textit{model fitting}  \cite{article:fitting_2010,article:fitting_2015}, i.e. deforming the reference shape according to a pre-existing model to resemble a given sample. The main difference between these two concepts is that restrictions to deformations are enforced by data (or specific shape characteristics) in model fitting and by a regularization parameter (independent of shape) in registration. Despite this close relationship, they are distinct branches of the literature since, initially, the raw data is not registered and correspondence must be achieved without access to a previous model (which can only be obtained upon registration). This originates complex pipelines, where different models and assumptions on the shape deformations are sequentially considered, for increasing refinement. A unified framework is desirable to achieve more principled approaches, as suggested in \cite{article:unified_approach}.
	
	We shall now formalize the previously stated conceptual problem and introduce the unified framework, as well as the motivation for its use in our setting. We consider a \textit{reference} $R = \{r_1,\dots ,r_{N_R}\} \subset \mathbb{R}^d$ and a \textit{target shape}  $S =  \{s_1,\dots,s_{N_S}\} \subset \mathbb{R}^d $, with a possibly different number of points $N_R$ and $N_S$, respectively. The former is a representative example of the shape being studied, as close as possible to any other shape, without missing data, outliers, or noise; while the latter is any sample of the dataset. In previous work \cite{article:UnsupervisedPipeline} with ear data scans, we have shown that most of the effects of translations and rotations can be previously removed; so it is here considered that the target shape and reference were already preprocessed. Additionally, scaling is kept in the model, as it represents differences in ear size. Consequently, only non-rigid deformations between the reference and a shape are explicitly modelled.
	
	Under these assumptions, any shape is obtained from a reference through a re-parametrization and a set of non-rigid deformations $\theta$. In the discrete setting, a re-parametrization can be represented by a permutation matrix $P$. Thus, the problem is formulated as finding the optimal deformations between $S$ and $R$
	\begin{equation}
		P^*, \theta^* \in \argmin_{P,\theta}  d(\mathcal{T}(PR,\theta),S) ,\label{eq:ProbForm}
	\end{equation} where $d(R,S)$ is a measure of dissimilarity quantifying the differences between two shapes and $\mathcal{T}(PR,\theta)$ is the transformed reference after application of the deformations $\theta$ and permutation $P$.
	
	The solution of \eqref{eq:ProbForm} provides a deformed reference whose points are in correspondence with the target shape. However, as stated above, this problem is generally not tackled in a unified way. On the registration side, a generic constraint is applied to the deformations, while $P$ is retrieved, disregarding any knowledge of the particular shape. On the fitting side, correspondence is assumed (for example by previously applying a registration method) or retrieved in a trivial manner (e.g. by taking the closest point). Both of these options entail that the correspondence is obtained with a different model $\mathcal{T}(PR,\theta)$ than the fitting. Looking at the two processes in the same setting may be beneficial as it allows both of them to benefit from additional information. In \cite{article:unified_approach}, the authors propose to formulate the unified framework through the use of Gaussian Processes (GP). The core idea is that by designing an appropriate kernel, prior beliefs about the shape structure can be incorporated independently from the registration algorithm, allowing the shape fitting and registration to leverage the same model.

	This framework has proven successful and has led to increasingly improved models of the human head \cite{article:Morphable_Blanz_Vetter,article:Gaussian_Process_Combining_3DMM}. However, developments have been made on the modelling side by increasing the size of training datasets \cite{article:3DMM_large_scale} or improving model combination \cite{article:Gaussian_Process_Combining_3DMM}, i.e., how to bring together models originating from different datasets. By applying the GP modelling to 3D ear point clouds, we have observed that indeed this is the most promising setting \cite{article:UnsupervisedPipeline}, but it calls for an improved outlook on the registration procedure. When models of the full head are considered, ears are a small detail that is often overlooked or disregarded \cite{article:Morphable_Blanz_Vetter,article:Registration_Face_tangencial_2021}, so a straightforward correspondence method is enough to provide acceptable accuracy, explaining the limited advances on this area. It is worth noticing that a recent approach within the GP framework extends the head model to include ears \cite{article:Gaussian_Process_Combining_3DMM_Ear}; however, it requires the identification of 50 manual landmarks for registration of the ear region. Besides, we approach the problem under a different perspective, where the ear scans are first reconstructed and the model is subsequently retrieved, while in \cite{article:Gaussian_Process_Combining_3DMM_Ear} a pre-existing head model is augmented with an additional ear model.
	
	On the other hand, the registration field has seen considerable improvements in recent years, and state-of-the-art methods are capable of dealing with increasingly larger ratios of outliers, noise and missing data \cite{article:review_reg,article:review_reg_2017}. Probabilistic approaches (a subclass of registration methods) have proven to be particularly suitable to this scenario. However, being decoupled from the modelling setting, they do not allow for extensive prior information regarding the shapes. This knowledge is usually limited to the expected overall magnitude of deformations applied to the reference, controlled by a set of algorithm parameters. Instead, modelling with kernels allows for the inclusion of additional prior beliefs in a straightforward and convenient manner. Upon the development of a fitting method with a generic kernel, the replacement with a tailored one is immediate and does not call for alterations to the method itself.

	\textit{Multi-annotation} is a popular concept in the machine learning setting, where a given data point is labelled by different sources (annotators), i.e., there is not a unique and true label assignment, but several possible ones. By modelling registration as a GP multi-annotator problem, we show that it is possible to perform probabilistic registration completely within the GP framework. This unified approach benefits both from a complex prior through the kernels and from the nice properties of probabilistic assignment, particularly when dealing with outliers and noise.

	\begin{figure}[hbt!]\centering
		\begin{subfigure}[t]{0.45\textwidth}\centering
			\includegraphics[clip, trim=6cm 8cm 7cm 1cm, scale=0.5]{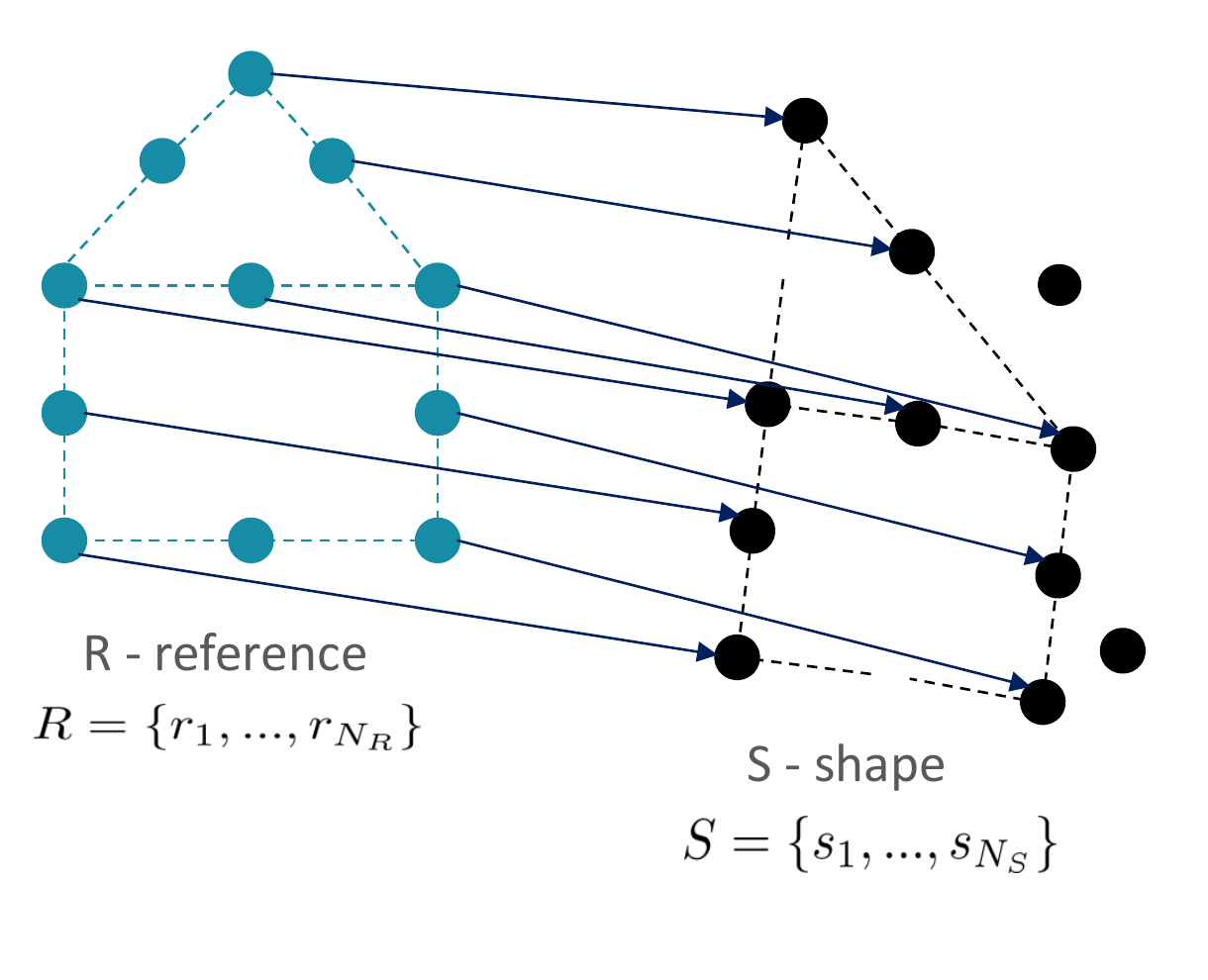}	\caption{Registration scheme}\label{subfig:scheme_reg}
		\end{subfigure}	
		\begin{subfigure}[t]{0.45\textwidth}\centering
			\includegraphics[clip, trim=6cm 1.4cm 8cm 11cm, scale=0.5]{scheme.pdf}\caption{Multi-annotator concept }\label{subfig:scheme_multi}
		\end{subfigure}	
		\caption{Schematic view of our method. The registration concept is illustrated on the left, where a reference point set $R$ is put into correspondence with a generic shape $S$. The main goal is to retrieve the deformation applied to the reference so that it resembles the target shape, as well as the matching between the two point sets. Note the existence of missing data, i.e., points found in $R$ but not in $S$, and outliers, i.e., points found in $S$ but not in $R$, introducing additional challenges to the registration procedure. On the right, the main concept behind our method is depicted. For each reference point with correspondence, $r_{Ci}$, several possible deformations are considered, referring to the different target shape points. Each deformation (label) has an associated variance $(\sigma_i^j)^2$ related to the level of confidence in that particular annotation.}
		\label{fig:scheme_full}
	\end{figure}

	\subsection{Related work}
	
	Throughout this section, we review closely related work under the two main areas of interest: registration within the GP framework and probabilistic approaches to the generic registration problem.

	\subsubsection{Registration within the GP framework}

	\paragraph*{Non-rigid registration with Iterative Closest Point}

	A proposed approach to tackle the problem registration problem/model fitting for the GP framework is a non-rigid application of the Iterative Closest Point (ICP) \cite{article:Registration_ICP}, where the transformation part is obtained through Gaussian Process Regression (GPR). This means that to each point in the reference we attribute the closest target point, based on their Euclidean distance. These correspondences are then taken as observations and GPR is used to compute deformations for the entire shape (the mean of the posterior is the reference used in the next iteration). Our approach relates to this method, in the formulation of the problem but not in the way the correspondences are retrieved. For ear shapes, given the large regions of missing data and the highly non-rigid deformations, the closest point approach leads to undesirable results \cite{article:UnsupervisedPipeline}, particularly on the bottom region, where the missing points from the posterior section cause the deformed reference to collapse.

	\paragraph*{Registration as an optimization problem}
	In another proposed approach in \cite{article:unified_approach}, the authors formulate the registration problem/model fitting for both surfaces and images. Here, we focus on the surface formulation. First, the authors do a low-rank approximation, obtaining a parametric approximation of the original kernel. The problem is then posed as a Maximum a Posteriori (MAP) estimation problem, where the likelihood expresses some distance measure between the target and reference shapes, and the prior is given by the GP. The authors chose the mean squared Euclidean distance from the reference to the closest target point and solve the problem with an L-BFGS optimizer \cite{article:L-BFGS}. Therefore, the restriction of hard-assignment when choosing correspondences is maintained in this approach, leading to similar problems as faced in the previous one.

	\subsubsection{Probabilistic registration}\label{sec:probareg}
	
	The previous approaches imply a deterministic attribution of correspondences between points, while a soft-assignment may improve robustness to noise and outliers \cite{article:review_reg}. 
	
	This leads, in the registration area, to probabilistic registration methods, of which the most used and representative is the Coherent Point Drift (CPD) \cite{article:Registration_CPD},  which considers the alignment of two sets as a probability density estimation problem. This approach takes $R$ (the reference) as a set of centroids coming from a Gaussian Mixture Model (GMM) and $S$ (any shape) as points generated by the centroids. An important detail is that the centroids are forced to move coherently as a group, thus preserving the topological structure of the points (motion coherence constraint over the velocity field). The goal is to estimate the centroid from which each point in $X$ was generated, thus resulting in a correspondence output. 
	
	While considered state-of-the-art, CPD still presents difficulties in overcoming a high incidence of outliers and missing data, as well as a different number of points between the reference and target. Consequently, variants of CPD have been developed in recent years to deal with such drawbacks by assigning different membership probabilities  \cite{article:CPD_PreservingGlobal_LocalStructures} or using  k-connected neighbours \cite{article:Registration_CPD_LocalConnectivity} to enforce the preservation of local structures. 
	Other variants are only applicable to rigid registration \cite{article:Registration_CPD_rigid_normals, article:Registration_CPD_rigid_local, article:Registration_CPD_rigid_robust} and, consequently, do not conform to our assumptions.
	
	An interesting recent work \cite{article:Registration_BCPD} proposes a Bayesian Formulation of CPD (BCPD). Under this setting, the authors guarantee convergence of the algorithm, introduce more interpretable parameters and reduce sensitivity to target rotation. Besides, this formulation is amenable to kernels beyond the Gaussian, thus presenting a close relationship to our work. In fact, we shall see that it is possible to establish a parallel between BCPD and our approach, under a given set of assumptions. Interestingly, in \cite{article:Registration_BCPD_PLUS} we see an improved version of BCPD, where GPR is used. However, note that the introduction of GPs has the single purpose of accelerating the algorithm. The point sets are initially subsampled, after which standard BCPD is conducted. The final step employs GPR to extend the retrieved deformations to the full shape.

	Finally, in \cite{article:comparison_method} the authors propose a probabilistic registration method, using a Point Distribution Model (PDM) as kernel, instead of the traditional Squared Exponential kernel used in \cite{article:Registration_CPD}. PDMs are the standard approach to retrieve statistical shape models from data in correspondence, through the application of Principal Component Analysis. They also propose the use of anisotropic Gaussian Mixture Models, oriented according to the surface normals, thus taking into account surface information. This method is closely related in the sense that it constitutes a probabilistic approach with an underlying shape modelling, thus merging soft-assignment with prior shape knowledge.
	
	\subsection{Our method}
	From the previous introduction, it is possible to conclude that proposals for registration within the GP framework target hard assignment and assume a one-to-one correspondence between a shape and reference, thus motivating their extension with a soft assignment. On the other hand, probabilistic registration methods do not consider detailed prior knowledge specific to the shape. This observation motivated us to develop Shape Fitting Gaussian Process (SFGP): a probabilistic shape fitting/registration method within the GP framework, where one can benefit both from a complex kernel prior and a soft assignment in the correspondences. A schematic view of the main idea behind our method is presented in Figure~\ref{fig:scheme_full}.

	Our main contributions are:
	\begin{itemize}
		\item \textbf{Shape registration/model fitting as a multi-annotator GPR.} We show how the problem of registration with soft assignment can be understood within the GP framework as a multi-annotator Gaussian Process Regression (Section~\ref{sec:multiannotator}).
		\item \textbf{Parallel between probabilistic registration and our method, SFGP.} We provide a parallel between BCPD and our algorithm, under a few assumptions, which allows us to benefit from the probabilistic setting (Section~\ref{sec:variance}). We further show how their differences lead to a good performance in the presence of extensive missing data (Section~\ref{sec:experiments}).
		\item \textbf{Application to a difficult registration problem -- 3D ears registration.} We show that our method is suitable for the registration of 3D point clouds with highly non-rigid deformations, high occurrence of missing data and outliers, by performing simulations with 3D point sets of human ears. The results show improvement with respect to state-of-the-art proposals (Section~\ref{sec:experiments}).
	\end{itemize}

	\section{Registration within the GP framework}
	\label{sec:multiannotator}
	
	In this section we present the formulation of our problem within the GP framework. This extends the approach in \cite{article:unified_approach} with the explicit modelling of missing data and the incorporation of multiple annotators to model probabilistic correspondences. 
	\subsection{Gaussian Processes}
	\label{sec:Background_GP_new}

	A Gaussian Process (GP) is a collection of random variables, any finite number of which have a joint Gaussian distribution. A GP $u(x)$ is fully specified by its mean $\mu(x)$ and covariance function $k(x,x')$ defined as 
	\begin{equation*}
		\begin{split}
			\mu(x) &= \mathbb{E} [u(x)]\\
			k(x,x') &=  \mathbb{E} [(u(x) - \mu(x))(u(x') - \mu(x'))]
		\end{split}
	\end{equation*} and usually written as 
	\begin{equation*}u(x) \sim \mathcal{GP}(\mu(x),k(x,x')).\label{GP_definition}
	\end{equation*} We refer to \cite{book:GP_MachineLearning} for a more thorough introduction to the theory of Gaussian Processes.
	
	\paragraph*{Outputs in higher dimensions}
	
	Gaussian Processes have initially been defined for scalar outputs, but they can be extended to the vector-case under certain assumptions so that the results obtained for the scalar case remain valid \cite{article:kernels_extension_to_vector}. In particular, a useful class of covariance functions for the vector-valued case arises from the scalar-valued covariance functions \cite{article:learning_vector_functions}. Let $A \in\mathbb{R}^{d\times d}$ be a symmetric, positive definite matrix and $l$ a real-valued covariance function. It can be shown that the matrix-valued function $K  \in\mathbb{R}^{d\times d}$ with entries $k_{ij}$ defined by $k_{ij} = A_{ij}l(x,x')$ is a valid covariance, with $A_{ij}$ representing the correlation between the $i$-th and $j$-th output component. Therefore, under the assumption that different dimensions have no correlation, it is possible to use any pre-existing scalar kernel and set $A$ as the identity matrix.

	\subsection{Problem formulation and notation}\label{subsec:prob_form} We consider that any shape $S =  \{s_1,...,s_{N_S}\}$ can be obtained from a reference shape $R = \{r_1,...,r_{N_R}\} $, where $s_i,r_i \in \mathbb{R}^d$. In particular, the shape $S$ is obtained by adding deformations $u(r)$ to the reference points,  where $u(r)$ is modelled as a Gaussian Process (GP) defined by a mean function $\mu: \mathbb{R}^d \to \mathbb{R}^d$ and a kernel $K: \mathbb{R}^d \times \mathbb{R}^d \to \mathbb{R}^{d\times d}$, and written as  $u(r) \sim GP(\mu(r),K(r,r'))$.

	Note that we are working in the case of non-scalar output, but as stated above the results obtained for scalar outputs may be applied, as long as the kernel $K$ is ensured to be valid. 
	We further assume that we can get noisy observations of the deformations 
	\begin{equation}
		\delta(r_i)= u(r_i)+\epsilon\label{eq:noisy_labels},
	\end{equation}
	where $\epsilon\sim \mathcal{N}(0,\sigma_n^2)$ and $\sigma_n^2$ is the noise variance.
	
	Under the traditional GP setting perspective, the reference points $r_i$ can be viewed as the input data, while the respective deformations $\delta(r_i)$ correspond to the labels or output variables.
	
	\paragraph*{Modelling missing data and outliers} In \cite{article:unified_approach} the authors assume a one-to-one correspondence between target shape and reference, so each point $s_i$ is attributed to a reference point as $S = \{r+u(r) \: | \: r \in R\} $. However, in challenging scenarios such as the ears, this assumption is far from true, given the large ratio of missing data and outliers found on the target shapes. Therefore, we consider the existence of reference points without correspondence in the target and vice-versa. We formulate this assumption by splitting the reference $R$ into two subsets $\{\: R_{C},\: R_{M}\}$, where the former set contains points with correspondence and the latter refers to missing data. In the same way, we split $S$ into the corresponding points and outliers as $S= \{\: S_{C},\: S_{O}\}$, such that  $S_{\textrm{C}}$ presents a one-to-one correspondence with $R_{C}$. Under these assumptions, the observed deformations can be expressed as
	\begin{equation}
		\delta=\begin{bmatrix}\delta_{C1} \\ \vdots \\ \delta_{CC}\end{bmatrix}  =\begin{bmatrix}s_{r_{C1}} - r_{C1}\\ \vdots \\ s_{r_{CC}}-r_{CC}\end{bmatrix},
		\label{eq:def}
	\end{equation}
	where $R_{C} = \{r_{C1},...,r_{CC}\}$ and $s_{r_{Ci}}$ is the shape point corresponding to the reference point $r_{Ci}$. A list of the notation used from this section onwards can be found in Table~\ref{tab:Notation}.

	\begin{table}[]\centering
		\begin{tabular}{p{6.5cm}p{8cm}}
			\textbf{Variable}                                & \textbf{Description}                                                                                                          \\ \hline
			$S =  \{s_1,...,s_{N_S}\} \subset \mathbb{R}^d $ & The target shape point set, with its respective vector representation $s =  (s_1^T,...,s_{N_S}^T)^T \in \mathbb{R}^{N_Sd} $   \\
			$R =  \{r_1,...,r_{N_R}\} \subset \mathbb{R}^d $ & The reference point set, with its respective vector representation $r =  (r_1^T,...,r_{N_R}^T)^T \in \mathbb{R}^{N_Rd} $ \\
			$R_{C} = \{r_{C1},...,r_{CC}\}$ & Set of reference points with correspondence to target, where $C$ is the number of points with correspondence  \\
			$S_{C} = \{s_{r_{C1}},...,s_{r_{CC}}\}$ & Set of shape points with correspondence to the reference, where $s_{r_{Ci}}$ is the target point corresponding to the reference point $r_{Ci}$          \\
			$\Delta= \{\delta_{C1}, \dots ,\delta_{CC} \}\subset \mathbb{R}^d $ & The deformations for each reference point with correspondence, with its respective vector representation $\delta =  (\delta_{C1}^T, \dots ,\delta_{CC}^T)^T \in \mathbb{R}^{N_Rd} $ \\
			$D_{\sigma^2_n} = \textrm{diag}(\sigma^2_{C1},\dots, \sigma^2_{CC})$  & Diagonal matrix of observation noise, where $\sigma^2_{C_i}$ is the variance of noise for observed deformation $\delta_{Ci}$ \\
			$\mathcal{C}_i = \{j : s_j \in S , p_{i,j}>P_{MIN} \}$  & Set of indices of target shape with correspondence with reference point $r_i$ \\
			$K = K_{RR} = [k(r_i,r_i) ]^{N_R}_{i=1} \in \mathbb{R}^{N_R d\times N_R d}$  & The kernel matrix of the entire reference point set, i.e. containing all the points in $R$ \\
			$K_{R_{C}R_{C}} = [k(r_{Ci},r_{Ci}) ]^{C}_{i=1} \in\mathbb{R}^{C d\times C d}$  & The kernel matrix of the reference points with correspondence \\	
			$I_d$  & The identity matrix of size $d$ \\
			$\omega$  & Outlier probability \\
			$D_{\varsigma^2} = \textrm{diag}(\varsigma^2_{1},\dots, \varsigma^2_{N_R})$  & The diagonal matrix of registration noise, where $\varsigma^2_{i}$ is the variance associated to reference point $r_i$ \\
			$\tilde{A} = A \otimes I_d$  & The Kronecker product of matrix $A$ with $I_d$ \\
			$1_N$  & The vector of ones with size $N$ \\
		\end{tabular}\caption{Notation table.}\label{tab:Notation}
	\end{table}

	\subsubsection{Registration problem}
	Within this framework, and because of the gaussianity of the distributions, for which the mode and mean coincide, the shape fitting and registration problem can be formulated as a MAP problem \cite{article:Gaussian_Process_2017_Vetter}, that is
	\begin{equation*}
		\max_u p\big(u \: | \: R,S\big).
	\end{equation*}
	In particular, the dependence is only on the observed points and deformations (not over the entire reference shape), so that 
	\begin{equation*}
		\max_u p\big(u \: | \: R_{C},\delta\big).
	\end{equation*}
	However, the correspondences are not known beforehand and they depend on the deformations themselves, leading to
	\begin{equation}
		\argmax_u p\big(u \: |\: R_{C}(u),\delta(u)\big).
	\end{equation}
	In an ICP-like approach, we split our problem into two, where $u$ is kept fixed in the inner maximization and $R_{C},\: \delta$ are kept fixed in the outer optimization. Consequently, the final formulation of the fitting problem is written as 
	\begin{equation}
		\max_u \Bigg\{\max_{R_{C},\delta} p\big(u \: |\: R_{C}(u),\delta(u)\big)\Bigg\}.\label{eq:Reg_Problem}
	\end{equation}
	In the outer problem, the current correspondence is used to estimate the transformations on the reference. Conversely, the inner maximization computes likely correspondences, given the currently transformed reference shape.
	
	\section{Formulation with multi-annotators}
	
	Given the conceptual formulation in \eqref{eq:Reg_Problem}, it is necessary to define the correspondence procedure to effectively solve the optimization. In Section~\ref{sec:corr_hard} we consider a hard-assignment, that is then extended through multi-annotators to a probabilistic one, in Section~\ref{sec:corr_soft}.
	
	\subsection{Correspondence with hard-assignment}\label{sec:corr_hard}
	
	Let us first consider the case where there is single correspondence for the reference inliers. Regardless of the method used to establish a matching, the output will be a vector of observed deformations $\delta$ with the same size as $R_{C}$, as expressed in \eqref{eq:def}. By \eqref{eq:Reg_Problem} and for fixed $R_{C}$ and $\delta$, we are interested in the MAP of the GP posterior. In particular, given a set of observed inputs $R_C$ and outputs $\delta$, the goal is to find the most likely deformations for unobserved points $R_M$. This corresponds to noisy Gaussian Process Regression (GPR), where $u$ is the GP and $R_C, \: \delta$ are the training dataset. In GPR, a Bayesian approach is followed to retrieve the predictive equations for test points (in this case $R_C$). Applied to our setting and for the full shape input $R$, these equations correspond to \cite{book:GP_MachineLearning}
	\begin{equation*}u_*| R_C,\delta, R \sim \mathcal{N}(\mu_p,\Sigma_P),
	\end{equation*}where $\mu_p\in \mathbb{R}^{N_Rd}$ and $\Sigma_p \in \mathbb{R}^{N_Rd \times N_Rd}$ are given as
	\begin{equation}
		\begin{split}
			\mu_p &= K_{R_{C}R}^T \big(K_{R_{C}R_{C}}+ \sigma^2_nI_{Cd} \big)^{-1}\delta\\
			\Sigma_p &= K_{RR}-  K_{R_{C}R}^T\big(K_{R_{C}R_{C}}+\sigma^2_nI_{Cd}\big)^{-1} K_{R_{C}R},
		\end{split}
		\label{eq:GPR_shape}
	\end{equation}
	where $K_{R_{C}R_{C}} = \big[k(r_{Ci},r_{Ci}) \big]^{C}_{i=1} \in \mathbb{R}^{Cd \times Cd}$ is the kernel matrix of the observed points, $K_{RR_C} = \big[k(r_j,r_{Ci}) \big]^{N_R,C}_{j,i=1} \in \mathbb{R}^{N_R\times C}$ is the kernel matrix of the predicted and observed points, $K_{R_CR} =K_{RR_C}^T  $ its transpose and $I_{Cd}$ is the identity matrix of size $Cd$. The deformed reference is then obtained as $\bar{r} = r + \mu_p $, to be used in the inner iteration for computation of the correspondences (e.g. by taking the closest point).
	
	\textbf{Remark on the kernel matrices}. Note that throughout the fitting procedure, only the points contained in $R$ are taken into account. Thus, we lie in the discrete setting and the prediction outputs $\mu_p$ and $\Sigma_p$ can be written as a vector and matrix, instead of functions. In the same way, $K_{RR}$, the kernel matrix of all reference points, is constant and all other kernel matrices ($K_{R_{C}R_{C}}, K_{R_{C}R}$) are subsets of $K_{RR}$. For this reason, $K_{RR}$ is simply denoted as $K$. Despite this observation, we note that the formulation with GPs (instead of its discrete counterpart, i.e., the multivariate normal distribution) is relevant, as the extension to a continuous surface is desired at later steps of the pipeline.

	\subsection{Introducing soft-assignment}\label{sec:corr_soft}
	Introducing a soft-assignment equates to attributing different possible target points to a reference point, each with a different probability of correspondence. In particular, we assume that a reference point $r_i$ has $N_S$ possible deformations, corresponding to a match with each target point. In the traditional machine learning setting, this idea can be expressed as multi-annotation. The concept refers to the multiple labellings of the same data point when there is no exact ground truth available (e.g. the attribution of a label is subjective). In this situation, a common approach is to obtain a collection of labels for each data point, provided by different annotators, with (possibly) different levels of confidence. The data points correspond to the reference points in $R$, while the labels correspond to the different possible deformations with respect to the target shape points. For a visual representation of the described model, we refer the reader to the respective graphical model in Figure~\ref{fig:graphical_model}.

	Therefore, we consider ${\delta}_{Ci}^j$ as the deformation between the reference point $r_{Ci}$ and the shape point $s_j$, with an associated variance $(\sigma_{Ci}^j)^2$. The final variance and deformation of point $r_i$ is then retrieved by weighting all of the $j$ contributions. By appropriately defining this relationship we are able to obtain predictive equations that follow the same structure as the GPR in \eqref{eq:GPR_shape}. Following \cite{article:GP_MultiLabel_Annotators}, we define these two quantities as 
	\begin{equation}
		\frac{1}{\sigma^2_{Ci}} = \sum_{j \in \mathcal{C}_i}  \frac{1}{(\sigma_{Ci}^j)^2} \quad,\quad 
		\hat{\delta}_{Ci} = \sigma^2_{Ci} \sum_{j \in \mathcal{C}_i} \frac{{\delta}_{Ci}^j}{(\sigma_{Ci}^j)^2}, \label{eq:updates_multi}
	\end{equation}
	where $\delta_{Ci}^j = s_j -\bar{r}_{Ci}$ . The variance $(\sigma_{Ci}^j)^2$, corresponding to the annotator's confidence in the original formulation, represents here the probability of correspondence between point $C_i$ and $s_j$. Note that while the variance of annotator $j$ is often considered the same across the data points $i$, here we consider that annotator $j$ of point $C_i$ is not necessarily the same as annotator $j$ of point $C_k$, i.e. there can be at most $N_R \times N_S$ annotators.

	The predictive equations can readily be obtained by assuming annotators provide the labels independently from each other, leading to the likelihood as 
	\begin{equation*}
		p(\delta|u) = \prod_{i=1}^{N_R}\prod_{j \in \mathcal{C}_i}\mathcal{N} \big(\delta_{Ci}^j|u_i,(\sigma_{Ci}^j)^2 \big).
	\end{equation*} Considering this expression and the prior on $u$, the predictive equations for an input $R$ follow  $\mathcal{N} (\hat{\mu}_p,\hat{\Sigma}_p)$, with mean and covariance \cite{article:GP_MultiLabel_Annotators}
	\begin{equation*}
		\begin{split}
			\hat{\mu}_p &=  K_{R_{C}R}^T(K_{R_{C}R_{C}}^{-1}+\tilde{D}_{\sigma^2_n})^{-1} \tilde{D}_{\sigma^2_n}\hat{\delta}\\
			\hat{\Sigma}_p &= K_{RR} - K_{R_{C}R}^T(K^{-1}_{R_{C}R_{C}}+\tilde{D}_{\sigma^2_n})^{-1} K_{R_{C}R},
		\end{split}
	\end{equation*}where $D_{\sigma^2_n} =\textrm{diag}(\sigma^2_{C1},\dots, \sigma^2_{CC})$, $\tilde{D}_{\sigma^2_n} =D_{\sigma^2_n} \otimes I_d$  and $\hat{\delta} = [\hat{\delta}_{C1},\dots,\hat{\delta}_{CC}]^T$, where $\sigma^2_{Ci}$ and $\hat{\delta}_{Ci}$ are given by \eqref{eq:updates_multi}. The kernel matrices are defined in \eqref{eq:GPR_shape}. We can see that these equations differ from \eqref{eq:GPR_shape} only on the noise covariance matrix and label vector, but follow exactly the same structure.
	
	In order to complete our algorithm, two points need to be addressed: how to define the correspondence set $R_{C}$ and how to compute $(\sigma_{Ci}^j)^2$. The computation of this value will be detailed in the next section.

	\begin{figure}
		\centering
		\begin{tikzpicture}
			\tikzset{vertex/.style = {shape=circle,draw,minimum size=2em}}
			\tikzset{vertex_new/.style = {shape=circle,draw,minimum size=0.6em}}
			\tikzset{vertex_obs/.style = {shape=rectangle,draw,minimum size=2em}}
			\tikzset{vertex_text/.style = {draw=white!80,minimum size=2em}}
			\tikzset{group/.style = {shape=rectangle,draw,minimum height=7.3em, minimum width=9.2em}}
			\tikzset{edge/.style = {->,> = latex'}}
			\node[vertex_obs,fill=blue!20] (x1) at  (0,0) {$r_{C1}$};
			\node[vertex] (f1) at  (0,2) {$u_{C1}$};
			\node[vertex_new] (fnew) at  (1,3.9) {$\sigma_{C1}^j$};
			\node[vertex_obs] (y1) at  (0,5) {$\delta_{C1}^j$};
			\node[group,text height=0.5em] (y1group) at  (0,4.5) {};
			\node[above right,text height=0.5em] at (y1group.south west) {$\mathcal{C}_1$};
			\node[vertex_text] (fdots) at  (2,2) {$\dots$};
			\node[vertex_obs,fill=blue!20] (x2) at  (4,0) {$r_{CC}$};
			\node[vertex] (f2) at  (4,2) {$u_{CC}$};
			\node[vertex_obs] (y2) at  (4,5) {$\delta_{CC}^j$};
			\node[vertex_new] (fnew2) at  (5,3.9) {$\sigma_{CC}^j$};
			\node[group,text height=0.5em] (y2group) at  (4,4.5) {};
			\node[above right,text height=0.5em] at (y2group.south west) {$\mathcal{C}_C$};
			\node[vertex_obs,fill=orange!10] (xM1) at  (7,0) {$r_{M1}$};
			\node[vertex] (fM1) at  (7,2) {$u_{M1}$};
			\node[vertex] (yM1) at  (7,4) {$y_{M1}$};
			\node[vertex_obs,fill=orange!10] (xM2) at  (11,0) {$r_{MM}$};
			\node[vertex] (fM2) at  (11,2) {$u_{MM}$};
			\node[vertex] (yM2) at  (11,4) {$y_{MM}$};
			\node[vertex_text] (fdots2) at  (9,2) {$\dots$};
			\draw[edge] (x1) to node [font=\small,left]{} (f1);
			\draw[edge] (f1) to node [font=\small,right]{} (y1);
			\draw[edge] (fnew) to node [font=\small,right]{} (y1);
			\draw[edge] (fnew2) to node [font=\small,right]{} (y2);
			\draw[edge] (x2) to node [font=\small,left]{} (f2);
			\draw[edge] (f2) to node [font=\small,right]{} (y2);
			\draw[edge] (xM1) to node [font=\small,left]{} (fM1);
			\draw[edge] (fM1) to node [font=\small,left]{} (yM1);
			\draw[edge] (xM2) to node [font=\small,left]{} (fM2);
			\draw[edge] (fM2) to node [font=\small,left]{} (yM2);
			\draw[very thick] (f1) edge[-] (fdots);
			\draw[very thick] (fdots) edge[-] (f2);
			\draw[very thick] (f2) edge[-] (fM1);
			\draw[very thick] (fM1) edge[-] (fdots2);
			\draw[very thick] (fdots2) edge[-] (fM2);
		\end{tikzpicture}
		\caption{Graphical model for the GP with multiple annotators in the context of shape modelling. The model follows the notation proposed in \cite{book:GP_MachineLearning} for GP graphical models, where the horizontal bold line represents a set of fully connected nodes. Squared nodes represent measurements (empty squares) or constants (filled squares) and circles are latent ones. Template points with correspondence are identified with blue, while missing points are identified with orange to facilitate visual recognition. The plates represent the repetition of contained variables and, in this case, refer to the multiple annotations $j$ of a given reference point. }\label{fig:graphical_model}
	\end{figure}
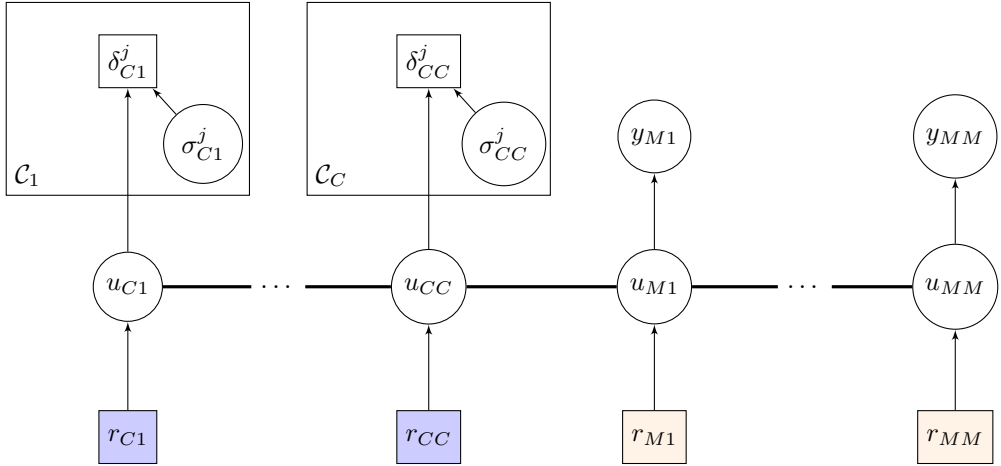
	
	\section{Computation of variance}
	\label{sec:variance}
	
	In order to obtain a theoretically sound update for elements $\sigma_{Ci}^j$ in \eqref{eq:updates_multi} we reformulate our problem in the standard probabilistic approach. This framework is first introduced in Section~\ref{sec:proba_intro} and then applied to our setting in Section~\ref{sec:proba_our}. In Section~\ref{sec:proba_VBI} we derive the parameter estimation for the previous formulation. Finally, in Section~\ref{sec:proba_VBI} we establish a parallel between the GP formulation in Section~\ref{sec:corr_soft} and the probabilistic approach presented in this section. We follow the Bayesian formulation proposed in \cite{article:Registration_BCPD}, instead of the original in \cite{article:Registration_CPD}, as the Bayesian setting has a strong connection to the Gaussian Processes. Throughout this section, we keep the notation introduced in Section~\ref{sec:multiannotator} whenever the variables refer to the exact same elements.
	
	\subsection{Probabilistic shape registration}\label{sec:proba_intro}

	Under the probabilistic formulation, a reference shape ($R$), upon an appropriate transformation, can be seen as a set of centroids of a Gaussian Mixture Model (GMM), where the target points of any shape ($S$) correspond to data generated by the centroids. Further, a point $s_j$ can be an outlier with probability $\omega$, in which case it is generated from an outlier probability distribution $p_{out}(s_j)$. If $s_j$ is not an outlier, then it corresponds to a point $r_i$ with probability $\alpha_i$ (membership probability).
	
	In \cite{article:Registration_BCPD} the authors consider explicit similarity transformations and non-rigid ones, such that a point $i$ of the transformed reference is given as 
	\begin{equation}
		\mathcal{T}(r_i) = \beta\Gamma(r_i+v_i) + \eta, \label{eq:bcpd_transformation_model}
	\end{equation}
	where $r_i$ is the original reference point, $\beta$ is a scale factor, $\Gamma$ is a rotation matrix, $\eta$ is a translation vector and $v_i$ is a displacement vector for non-rigid transformations. We use $\rho = (\beta,\Gamma,\eta)$ to denote the set of similarity transformations. According to GMM, the generation of a target point $s_j$ follows a multivariate normal distribution with mean $\mathcal{T}(r_i)$ -- the transformed reference point -- and covariance matrix $\varsigma^2 I_d$, so the probability distribution to generate $s_j$ starting from $r_i$ is
	\begin{equation} \phi_{ij}(s_j;\mathcal{T}(r_i) ,\varsigma^2)= \frac{1}{(\varsigma \sqrt{2\pi})^d} \exp{ \Big(  -\frac{\|s_j - \mathcal{T}(r_i) \|^2}{2\varsigma^2}  \Big)}.  \label{eq:bcpd_GMM}
	\end{equation}
	In order to explicitly introduce correspondences, two aditional variables are added: $c \in \{0,1\}^{N_S}$, an indicator variable that takes value of $1$ for $c_j$ if point $s_j$ is an outlier, and $e \in \{1,\dots, N_R\}^{N_S}$, where $e_j = i$ if the $j$-th target point corresponds to the $i$-th reference point. Taking the outlier assumptions into account, we obtain the joint distribution for $(s_j,e_j,c_j)$ as
	\begin{equation}
		p(s_j,e_j,c_j|R,\varsigma^2,v, \rho) =\{w p_{out}(s_j)\}^{1-c_j} \Big\{(1-w) \prod_{i=1}^{N_R}(\alpha_i\phi_{ij})^{\gamma_i(e_j)} \Big\}^{c_j}, \label{eq:bcpd_joint_target_dist}
	\end{equation}
	where $\gamma_i$ is an indicator function, taking a value of $1$ if $e_j = i$ and $0$ otherwise, and $\alpha_i$ is the probability that $e_j = i$, with $\sum_{i=1}^{N_R} \alpha_i = 1$. The authors take $p(\alpha)$ as a Dirichlet distribution and set a prior on the deformations as 
	\begin{equation*}
		p(v|R) = \phi(\delta;0,\lambda^{-1} G \otimes I_d),
	\end{equation*} 
	where $G = (g_{ii'}) \in \mathbb{R}^{N_R \times N_R} $, with $g_{ii'} = k(r_i,r_i')$ and $k(x,x')$ a kernel function; $\lambda$ is a positive constant.
	
	Finally, the full joint is given as 
	\begin{equation}
		p(S,T,\theta) \propto p(\delta|R)p(\alpha)\prod_{j=1}^{N_S} p(s_j,e_j,c_j|R,v,\varsigma^2,\alpha, \rho),  \label{eq:bcpd_full_joint}
	\end{equation}
	where $\theta=(v,\varsigma^2,\alpha,\rho,c,e)$.
	
	\subsection{Formulation}\label{sec:proba_our}
	We shall briefly detail the assumptions and problem formulation used in this approach, by deriving equivalents of equation \eqref{eq:bcpd_transformation_model} through \eqref{eq:bcpd_full_joint} according to our assumptions. This will lead to the final expression for the joint distribution in \eqref{eq:joint_distribution}.

	\paragraph{Transformation model}
	We do not consider similarity transformations, so the transformation acting on the reference is merely given by the displacement vector $v_i$
	\begin{equation}
		\mathcal{T}_i = r_i + v_i.
	\end{equation}
	\paragraph{Gaussian mixture model}
	According to the GMM we obtain a similar expression to \eqref{eq:bcpd_GMM}
	\begin{equation} \phi_{ij}(s_j;\mathcal{T}(r_i),\varsigma_i^2)= \frac{1}{(\varsigma_i \sqrt{2\pi})^d} \exp{ \Big(  -\frac{\|s_j - \mathcal{T}(r_i)\|^2}{2\varsigma_i^2}  \Big)},
	\end{equation} but we introduce an individual variance for each reference point given as $\varsigma_i^2$. In Section~\ref{sec:experiments} it will become clear that this has a positive impact when dealing with large regions of missing data. Regarding the outliers, we follow the same assumptions as in BCPD, but for simplification, we take $p_{out}(s_j) = 1/N_S$, as was previously taken in CPD. Besides, we take equal membership probabilities $\alpha_i = 1/N_R$, meaning that a point in the target is expected to be associated with any point of the reference with equal probability. This is also the assumption in CPD and is here taken for simplification on this first formulation of the framework. Thus, we obtain the joint distribution
	\begin{equation}
		p(s_j,e_j,c_j|R,v,\varsigma^2) =\Big\{\frac{w}{N_S} \Big\}^{1-c_j} \Big\{\frac{(1-w)}{N_R} \prod_{i=1}^{N_R}(\phi_{ij})^{\gamma_i(e_j)} \Big\}^{c_j}.
	\end{equation}
	\paragraph{Prior distributions}
	The prior on deformations is expressed with the previously defined kernel $k(r,r')$, except that $\lambda$ is taken as $1$, since it can be included within the kernel. Therefore, $\lambda^{-1} G \otimes I_d$ can simply be denoted as $K(r,r')$, as defined in \ref{subsec:prob_form} and the prior is given as $p(v|r) = \mathcal{N}(0,K)$.
	
	\subsubsection{Full joint distribution}
	Finally, the full joint distribution is obtained as
	\begin{equation}
		p(s,r,\theta) \propto p(v|r)\prod_{j=1}^{N_S} p(s_j,e_j,c_j|r,v,\varsigma^2),
		\label{eq:joint_distribution}
	\end{equation}
	where $\theta=(v,\varsigma^2,c,e)$ are the parameters to be estimated.
	
	\subsection{Solving the problem with Variational Bayesian Inference}\label{sec:proba_VBI}
	In order to estimate the parameters $\theta$ in \eqref{eq:joint_distribution}, we resort to Variational Bayesian Inference (VBI) \cite{article:VBI,article:VBI_2}, a useful tool when dealing with challenging posterior distributions. We start by providing an overview of VBI and its formulation for out problem, followed by the update equations obtained through this approach.
	
	\subsubsection{Background}
	The idea behind VBI is to use a distribution $q(\theta)$ to approximate the true posterior $p(\theta|S,R)$, where the closeness between the two distributions is measured by the Kullback–Leibler (KL) divergence. Therefore, the goal is to minimize the KL divergence between $q$ and $p$, i.e.
	\begin{equation*}\begin{split}
			q^*(\theta) &= \argmin_{q(\theta \in Q)}  \quad KL(q(\theta)  \enskip || \enskip p(\theta|S,R) )\\
			&= \argmin_{q(\theta \in Q)}  \mathbb{E}[\log q(\theta)] + \mathbb{E}[\log p(\theta,R,S)] +\log p(S,R), \\
	\end{split}\end{equation*} where $Q$ is a predefined set of distribution families to which $q$ belongs. However, since $\log p(S,R)$ may not be computable, the evidence lower bound (ELBO) is maximized instead 
	\begin{equation*}ELBO(q) = \mathbb{E}[\log p(\theta,R,S)] - \mathbb{E}[\log q(\theta)].
	\end{equation*}The ELBO is equivalent to the negative KL divergence up to a constant, and therefore maximizing the former is equivalent to minimizing the latter. The choice of a suitable form for $Q$ is fundamental, as it should ideally lead to a sufficiently simple ELBO, while being flexible enough to provide a good approximation to the original posterior distribution. 
	
	Here, we assume that $q$ has a strong separation form, i.e. the latent variables are mutually independent and governed by different factors. This leads to $q(\theta) = \prod_{i=1}^M q_i(\theta_i)$, where $q_i(\theta_i)$ is the distribution for the variable $\theta_i$. In particular, we consider
	\begin{equation*}
		q(\theta) = q_1(v)q_2(c,e)q_{31}(\varsigma_1^2)\dots  q_{3i}(\varsigma_i^2) \dots q_{3N_R}(\varsigma_{N_R}^2).
	\end{equation*} A standard method to maximize the ELBO, and the one followed here, is the Coordinate Ascent Variational Inference. If we fix all other $q_j$, then we know that the optimal $q_i$ is 
	\begin{equation}
		q_i(\theta_i)^* \propto \exp \{ \mathbb{E}_{-i} [\log p(\theta_i | \theta_{-i},S,R ) ] \}  \propto \exp \{ \mathbb{E}_{-i} [\log p(\theta_i , \theta_{-i},S,R ) ] \},
	\end{equation} where $\mathbb{E}_{-i} [\log p(\theta,S,R ) ]$ is the expectation of the joint probability with respect to the remaining $q_{j\neq i}$ and $\theta_{-i}$ corresponds to all parameters in $\theta$ except $\theta_i$. Hence, each $q_i$ is updated iteratively by computing  $\mathbb{E}_{-i} [\log p(\theta,S,R ) ]$, until convergence is reached.

	\subsubsection{Update equations} 
	\label{sec:update_equations}
	
	We present the updates for each component of $q(\theta)$ in Proposition~\ref{prop:updates_q1}, Proposition~\ref{prop:updates_q2} and Proposition~\ref{prop:updates_q3}. The proofs follow \cite{article:Registration_BCPD, article:VBI} and can be found in Supplementary Material. The final equations exhibit a similar structure to those in \cite{article:Registration_BCPD}, except that $\varsigma^2$, taken as a scalar in \cite{article:Registration_BCPD}, is replaced by the diagonal matrix $D_{\varsigma^2}$. 
	
	For ease of notation, and in preparation for the subsequent equations, we define $p_{ij} = \mathbb{E}[c_j\gamma_i(e_j)]$ as the probability of correspondence between reference point $i$ and target point $j$, with the respective probability matrix $P = [p_{ij}]^{N_R,N_S}_{i,j=1} \in [0,1]^{N_R \times N_S}$. We further define $\nu_i =  \sum_{j=1}^{N_S}  p_{ij} $, representing the expected number of target points corresponding with $r_i$, as well as $\nu =P 1_{N_S}$, the corresponding vector.
	
	\begin{prop} The deformations $v$ follow a normal distribution $\mathcal{N}(\mu_v,\Sigma_v)$, with the update equations for mean and covariance given as
		\begin{equation}\begin{split}
				\mu_v &=  \Sigma_v\tilde{D}_\nu \tilde{D}_{\varsigma^2}^{-1} (\tilde{D}_{\nu}^{-1}\tilde{P}s-r) \\
				\Sigma_v &=  (K^{-1} + \tilde{D}_\nu \tilde{D}_{\varsigma^2}^{-1} )^{-1},
			\end{split}\label{eq:BCPD_update_def}\end{equation}
		where $D_{\varsigma^2} = \textup{diag}(\varsigma^2_{1},\dots, \varsigma^2_{N_R})$ and $D_{\nu} = \textup{diag}(\nu_{1},\dots, \nu_{N_R})$. 
		\label{prop:updates_q1}
	\end{prop} 
	\begin{prop} The update for the correspondence probability is
		\begin{equation}\label{eq:proba_corr}
			p_{ij} =  \frac{(1-w)\langle\phi_{ij}\rangle}{\frac{N_R}{N_S} w  + (1-w) \sum_{i'=1}^{N_R} \langle\phi_{i'j}\rangle},
		\end{equation}
		where $	\langle\phi_{ij}\rangle =\phi_{ij}(s_j; \mathcal{T}(r_i),\varsigma_i^2) \exp\left\{-\frac{\textup{Tr}(	\Sigma_v^i)}{2\varsigma_i^2} \right\} $, $ \mathcal{T}(r_i)= r+ \mu_v$, $\Sigma_v^i$ is the submatrix of $\Sigma_v$ related to the $v_i$ component and $\textup{Tr}(\cdot)$ is the trace of the matrix.
		\label{prop:updates_q2}
	\end{prop}
	\begin{prop} The update for each variance term $\varsigma_i^2$ is given as 
		\begin{equation}
			\varsigma_i^2 = \frac{1}{ d}\Big( \frac{ [ \tilde{P}\textup{diag}(s)s]_i -2\bar{r}_i^T [ \tilde{P}s ]_i  }{\nu_i}   +  \|\bar{r}_i\|^2 + \textup{Tr}(\Sigma_v^i) \Big),
		\end{equation}where $[A]_i$ refers to the $i$-th row of matrix $A$.
		\label{prop:updates_q3}
	\end{prop}
	
	\subsection{Parallel with GP framework}
	
	If we assume that there are no missing points, then it is possible to establish a parallel between the previous formulation and the GP framework. With this aim, we reformulate our expressions in order to obtain a similar structure to Propositions \ref{prop:updates_q1} through \ref{prop:updates_q3}. Note that the update step of $p_{ij}$ can be understood as the "getting correspondence" part, i.e. the inner optimization of Problem~\eqref{eq:Reg_Problem}. 
	
	\begin{prop} Considering no missing points, i.e. $R_{C} = R$, and if the variance $(\sigma_{Ci}^j)^2$ in Equation \eqref{eq:updates_multi} is taken as
		\begin{equation}
			(\sigma_{Ci}^j)^2 = \frac{\varsigma_i^2}{p_{ij}}, \label{eq:var_parallel}
		\end{equation} where $p_{ij}$ is given by \eqref{eq:proba_corr}, then an equivalent exists between the update equations in Proposition~\ref{prop:updates_q1}, Proposition~\ref{prop:updates_q2}, Proposition~\ref{prop:updates_q3} and the update equations for multi-annotator GPR in \eqref{eq:GPR_shape}.
		\label{prop:variance}
	\end{prop}
	
	\begin{proof}
		When $R_{C} = R$, and since $K^T=K$, the posterior mean and covariance in \eqref{eq:GPR_shape} become
		\begin{equation*}
			\begin{split}
				\mu_p &= K(K+ \tilde{D}_{\sigma^2_n})^{-1}\hat{\delta}\\
				D_{\sigma_P^2} &= K-  K(K+\tilde{D}_{\sigma^2_n})^{-1} K,
			\end{split}\label{eq:posterior_no_miss}
		\end{equation*} where $D_{\sigma^2_n} =\textrm{diag}(\sigma^2_{1},\dots, \sigma^2_{N_R})$, with $\sigma^2_{i}$ and $ \hat{\delta}_{i}$ given by \eqref{eq:updates_multi} and here restated without the notation for correspondences (as there are no missing points)
		\begin{equation*}
			\frac{1}{\sigma^2_{i}} = \sum_{j \in N_S}  \frac{1}{(\sigma_{i}^j)^2}, \quad \hat{\delta}{i} = \sigma^2_{i} \sum_{j \in N_S} \frac{\delta_{i}^j}{(\sigma_{i}^j)^2}.
		\end{equation*} Taking the variance as in \eqref{eq:var_parallel}, we can write the previous equations as
		\begin{equation*}\begin{split}
				\frac{1}{\sigma^2_{i}} &= \frac{P_i 1_{N_s}}{\varsigma_i^2} = \frac{\nu_i}{\varsigma_i^2}\\
				\hat{\delta}_{i}  & =\nu_i^{-1}\sum_j p_{ij} (s_j  -r_i) = \nu_i^{-1} \tilde{P}_i s - \nu_i^{-1}\sum_j p_{ij}  r_i = \nu_i^{-1} \tilde{P}_i s - r_i ,
			\end{split}\label{eq:ann_input}
		\end{equation*}
		where $P_i$ refers to the $i$-th row of matrix $P$. Therefore, we have that $D_{\sigma^2_n} = D_{\varsigma^2} D_{\nu}^{-1}$ and $\hat{\delta} =  \tilde{D}_{\nu}^{-1}\tilde{P}s-r$. The posterior deformations  in \eqref{eq:posterior_no_miss} can then be written as 
		\begin{equation*}
			\begin{split}
				\mu_p &= K(K+ \tilde{D}_{\varsigma^2} \tilde{D}_{\nu}^{-1})^{-1}\delta\\
				&= KK^{-1}\Big[1+  \tilde{D}_{\varsigma^2} \tilde{D}_{\nu}^{-1}K^{-1}\Big]^{-1}\delta\\
				&=  \Big[ K^{-1} + \tilde{D}_{\nu}\tilde{D}_{\varsigma^2}^{-1} \Big]^{-1}  \tilde{D}_{\nu} \tilde{D}_{\varsigma^2}^{-1} ( \tilde{D}_{\nu}^{-1}\tilde{P}s-r)\\
			\end{split}\label{eq:final_ourdef}
		\end{equation*} and the covariance as \begin{equation*}
			\begin{split}
				D_{\sigma_P^2} &=K-  K(K+\tilde{D}_{\varsigma^2} \tilde{D}_{\nu}^{-1})^{-1} K\\
				&= K - \Big[ K^{-1} + \tilde{D}_{\nu}\tilde{D}_{\varsigma^2}^{-1} \Big]^{-1}  \tilde{D}_{\nu} \tilde{D}_{\varsigma^2}^{-1}K \\
				&= \Big[ K^{-1} + D_{\nu}D_{\varsigma^2}^{-1} \Big]^{-1} \Big[ K(K^{-1} + \tilde{D}_{\nu}\tilde{D}_{\varsigma^2}^{-1}) - \tilde{D}_{\nu}\tilde{D}_{\varsigma^2}^{-1}K \Big] \\
				&= \Big[ K^{-1} + \tilde{D}_{\nu}\tilde{D}_{\varsigma^2}^{-1} \Big]^{-1},
			\end{split}\label{eq:final_cov}
		\end{equation*} thus being equivalent to the expressions for $\mu_v$ and $\Sigma_v$ in \eqref{eq:BCPD_update_def}. This entails that $v$ and $\delta$, under the assumption of no missing data, refer to the same variable and are updated with equivalent equations. Therefore, if $p_{ij}$ and $\varsigma_i^2$ are updated according to Proposition~\ref{prop:updates_q2} and Proposition~\ref{prop:updates_q3} (respectively), there is an equivalence between the two methods.
	\end{proof}
	Although this is established for the case of no missing data, we take \eqref{eq:var_parallel} as a reasonable update for the annotators variance, together with the necessary updates for $p_{ij}$ and $\varsigma_i^2$.

	\subsection{Missing data points}
	
	Given the probability matrix $P$, we apply a predefined threshold $P_{MIN}$, such that pairings with a lower value than $P_{MIN}$ are identified as non-corresponding. So, for each point $r_i$, the considered correspondences to the target are $\mathcal{C}_i = \{j : s_j \in S , p_{ij}>P_{MIN} \}$. Then if a point $r_i$ has no elements in $\mathcal{C}_i$, it is considered a missing point, meaning that  $R_{M} = \{r_i: r_i \in R, |\mathcal{C}_i|=0  \}$ and $R_{C} = \{r_i: r_i \in R, |\mathcal{C}_i|>0  \}$.

	\subsection{Algorithm for SFGP}\label{app:algorithm}
	
	The pseudo-code for our method is found in Algorithm~\ref{alg:RegFull} and Algorithm~\ref{alg:GetCorr}, where the former contains the main outer steps and the latter details the computation for the correspondence part.
	\begin{algorithm}[hbt!]
		\begin{algorithmic}[1]
			\REQUIRE{$r, s, K, D_{\varsigma_0^2},\omega, P_{MIN}$} 
			\STATE{$\bar{r} = r$, $D_{\sigma_P^2} = \textbf{0} $}
			\WHILE{some stopping criterion is not met}
			\STATE{$R_{C},\hat{\delta},D_{\sigma^2_n} =\textrm{get\_correspondences} (r,s,\bar{r},D_{\sigma_P^2},D_{\varsigma^2}, \omega, P_{MIN})$}\label{algStep:step_getcorr}
			\STATE{$\mu_p = K_{R_{C}R}^T(K_{R_{C}R_{C}}+ \tilde{D}_{\sigma^2_n})^{-1}\hat{\delta}$}
			\STATE{$D_{\sigma_P^2} = K_{RR}-  K_{R_{C}R}^T(K_{R_{C}R_{C}}+\tilde{D}_{\sigma^2_n})^{-1} K_{R_{C}R}$}
			\STATE{$\bar{r} = r + \mu_p$}
			\STATE{$\varsigma_i^2 = \frac{1}{ d}\Big( \frac{ [ \tilde{P}\textrm{diag}(s)s]_i -2\bar{r}_i^T [ \tilde{P}s ]_i  }{\nu_i}   +  \|\bar{r}_i\|^2 + \textrm{Tr}(D_{\sigma_P^2}) \Big)$}
			\ENDWHILE
		\end{algorithmic}
		\caption{SFGP}
		\label{alg:RegFull}
	\end{algorithm}
	\begin{algorithm}[hbt!]
		\begin{algorithmic}[1]
			\REQUIRE{$r,s,\bar{r},D_{\sigma_P^2},D_{\varsigma^2}, \omega, P_{MIN}$}
			\ENSURE{$D_{\sigma^2_n} = \textrm{diag}(\sigma^2_{C1},\dots, \sigma^2_{CC}),  \hat{\delta} =  (\hat{\delta}_{C1}^T, \dots ,\hat{\delta}_{CC}^T)^T, R_{C}$}
			\FOR{$(i, j) \gets (1, 1)$ to $(N_R, N_S)$} 
			\STATE{$\phi_{ij}(s_j;\bar{r}_i,\varsigma^2)= \frac{1}{(\varsigma_i \sqrt{2\pi})^d} \exp{ \Big(  -\frac{\|s_j - \bar{r}_i\|^2}{2\varsigma_i^2}  \Big)}$}
			\STATE{$\langle \phi_{ij} \rangle = \phi_{ij}\exp\Big\{-\frac{1}{2\varsigma^2}Tr(\sigma_{P_i}^2 I_d)\Big\}$}
			\STATE{	$	p_{ij} =  \frac{(1-w)\langle\phi_{ij}\rangle}{\frac{N_R}{N_S} w  + (1-w) \sum_{i'=1}^{N_R} \langle\phi_{i'j}\rangle}  $}
			\STATE{$(\sigma_{i}^j)^2 = \frac{\varsigma_i^2}{p_{ij}}$}
			\ENDFOR
			\STATE{$\mathcal{C}_i = \{j : s_j \in S , p_{ij}>P_{MIN} \}$\par$R_{C} = \{r_i: r_i \in R, |\mathcal{C}_i|>0  \}$}
			\FOR{$i \in R_{C}$} 			
			\STATE{$\frac{1}{\sigma^2_{Ci}} = \sum_{j \in \mathcal{C}_i}  \frac{1}{(\sigma_{Ci}^j)^2} $}
			\STATE{$ \hat{\delta}_{Ci} = \sigma^2_{Ci} \sum_{j \in \mathcal{C}_i} \frac{\delta_{Ci}^j}{(\sigma_{Ci}^j)^2}$}	
			\ENDFOR
		\end{algorithmic}
		\caption{get\_correspondences}
		\label{alg:GetCorr}
	\end{algorithm}
	
	\section{Experimental results and discussion} 
	\label{sec:experiments}

	In this section we present the results of experiments with both 2D and 3D data, with the respective discussion. For each subsection, we first describe the datasets and settings, following with an analysis of the results.
	
	\subsection{2D data}\label{subsec:Exp_2d_fish}
	
	\subsubsection{Dataset}
	As 2D data, we take the \textit{Fish Dataset} \cite{article:TPS-RPM}, where the reference is a 2D fish with 98 points. The target point sets are then generated by applying different kinds of alterations to the data. Non-rigid deformations are generated by warping the reference points with a Gaussian radial basis function. The dataset has four other variations considering outliers, missing data, rotation and noise, all of them with a moderate level of deformations included\footnote{A more detailed description of the dataset can be found in \cite{article:TPS-RPM}}. In order to accurately replicate the ear data challenges, we further create a new dataset, based on the noise level 2 of the \textit{Fish Dataset}. Here, we introduce structured missing data in the following way: we choose one point of the reference as centre and increasingly set the width of a squared bounding box around this point --- all the points within the box are removed. 
	
	\subsubsection{Setting}
	\begin{table}[]	\centering
		\begin{tabular}{p{3.5cm}p{11cm}}
			\textbf{Name}                                & \textbf{Description}                                                                                                          \\ \hline
			\textit{SFGP\_Full} & SFGP in its complete version  \\
			\textit{SFGP\_bcpdReg} & SFGP where the registration variance $\varsigma^2$ is taken as a scalar instead of a vector, computed according to BCPD equations  \\
			\textit{GPReg\_noTresh} & SFGP without the threshold for missing points $P_{MIN}$  \\
			\textit{GPClosestPnt} & registration with GPR, but where the correspondence part is achieved by taking the closest point, i.e. not considering multi-annotators  \\
			\textit{BCPD\_Standard} &  BCPD method with the standard parameters \\
			\textit{BCPD\_Opt\_Norm} &   BCPD method with optimized parameters for the \textit{Fish Dataset}, with normalization of both shapes \\
			\textit{BCPD\_Opt\_noNorm} &  BCPD method with optimized parameters for the \textit{Fish Dataset}, without normalization of both shapes, since this is not used in our method and could potentially benefit it in some cases. \\
		\end{tabular}\caption{Brief description of the different methods used in the experiments. For BCPD, the absence of normalization means that the shapes maintain their relative size, i.e., they are both normalized with respect to the target shape size, as recommended by the authors.}\label{tab:methods_exp}
	\end{table}
	
	We consider different variations of our method, as well as different variations of BCPD, in order to show the relevance of each modification. Their description can be found in Table~\ref{tab:methods_exp}. To fairly compare our method with BCPD, we set their parameters with the same value whenever possible --- consequently, we use the Squared Exponential kernel for our model. The remaining parameters and initial values for VBI are tuned with the deformation level 1 for both methods, by grid search. A detailed description of all settings can be found in the Supplementary Material.

	\subsubsection{Metrics} 
	For the evaluation of results, we mainly look at the Euclidean distance error between corresponding deformed reference $\bar{r}_i$ points and the ground truth $s_i^*$, i.e. the complete and deformed target shape without noise, averaged over the shape, so $d(s,t) = \frac{1}{N_R}\sum_{i=1}^{N_R} \| s_i^* - \bar{r}_i \|_2^2$. This is then averaged over the entire dataset, consisting of 100 samples. However, it should be noted that BCPD will occasionally not lead to a successful registration, in which case it does not produce an output or does not produce correspondence for any point. Since this result will not be taken into account for the distance metric and often occurs in the most challenging settings, we also present the fraction of successful registration. Our method does not consider a failed registration unless there are no deformations found in the first iteration.

	\subsubsection{Discussion} 
	The results for all the considered methods and data variations can be found in Figure~\ref{fig:Fish_Tests} and Figure~\ref{fig:Fish_Tests_2}. Our main focus is the dataset with an increasing level of missing regions (Figure~\ref{subfig:Missing}), as this closely replicates the challenges in the ear reconstruction problem. While for the lowest level it is evident that BCPD (when optimized) performs better, as we increase the missing area, our method presents a progressive advantage. Comparing \textit{SFGP\_Full}, \textit{SFGP\_bcpdReg} and \textit{SFGP\_noTresh} it becomes clear why those modifications are advantageous when facing extensive missing regions. It is also evident that the closest point approach has the poorest performance overall. 
	
	It is also interesting to look at the results in the presence of outliers (Figure~\ref{subfig:Out1} and Figure~\ref{subfig:Out3}), for which we tested all methods with $\omega = 0.1$ and $\omega = 0.3$ (with the exception of \textit{GPClosestPnt} where this is not applicable), since this parameter reflects the expected outlier probability. While BCPD outperforms SFGP when $\omega$ is adequately adjusted to the real outlier occurrence, we note that the behaviour of our method is not as dependent on this parameter. Thus, in the absence of prior knowledge, SFGP is found to be a more suitable choice.  
	
	Looking at the variation of noise (Figure~\ref{subfig:Noise}) and deformations (Figure~\ref{subfig:Def}), we see that overall an adequately fitted BCPD outperforms our method and is able to achieve lower errors, even when both parameters are previously tuned. We also note that the two variations of our method always perform better than the full proposal for these scenarios --- the proposed alterations do not bring an advantage when we are not dealing with structured missing data. However, this decrease in performance is deemed acceptable given the gain it provides in Figure~\ref{subfig:Missing} and when compared with the Closest Point proposal always leads to lower error.
	
	An intuition on why our method is able to cope well with extensive missing data is offered in Figure~\ref{fig:Fish_Example}, where we compare the fitting results from BCPD with different levels of deformation and our proposed method. It is clear that the main challenge in achieving an adequate fitting with BCPD is that a high level of deformations leads to the collapse of the missing regions, while lower values do not provide enough flexibility to fit small details found in the non-missing parts. With SFGP, collapsing is prevented, while allowing enough non-rigid deformation to accurately fit fine details.
	
	Furthermore, in Figure~\ref{fig:Missing_details} we present additional metrics for the missing region version of this dataset. The high performance on both recall and precision presented by SFGP tells us that the lower distance error previously observed relates to an accurate identification of missing points. Additionally, we notice that the GP with closest point correspondence exhibits a very high precision, but at the cost of a low recall. 
	
	\begin{figure}[hbt!]
		\centering
		\begin{subfigure}[t]{\linewidth}\centering
			\includegraphics[scale=0.5]{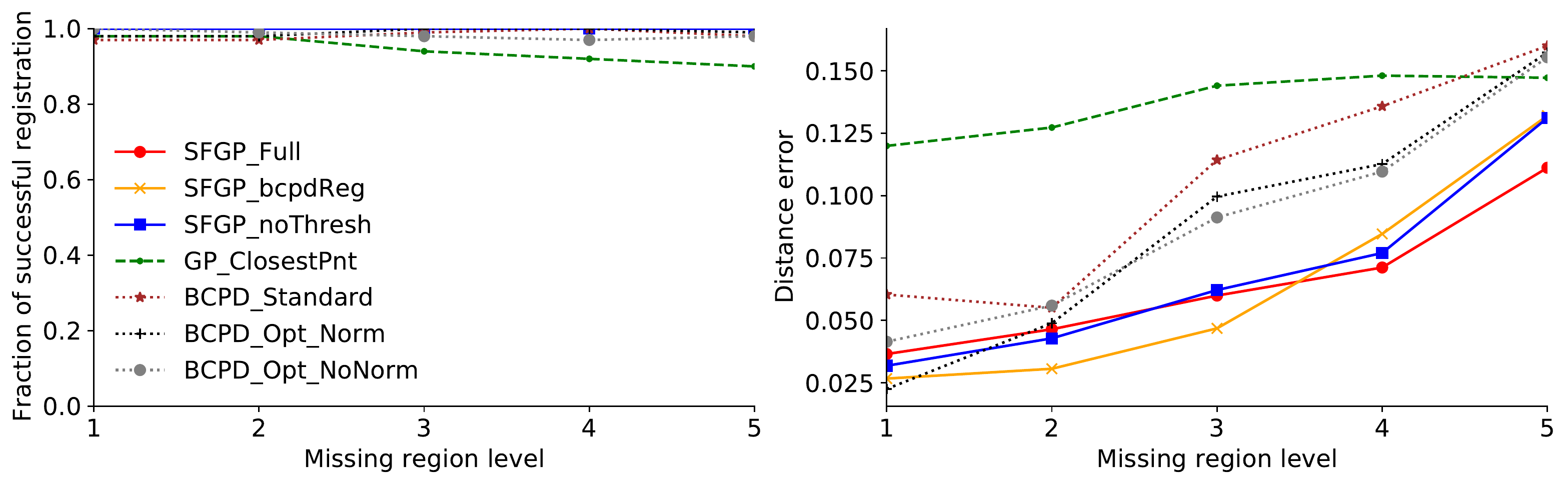}\caption{Increasing level of missing region}\label{subfig:Missing}
		\end{subfigure}
		
		\begin{subfigure}[t]{\linewidth}\centering
			\includegraphics[scale=0.5]{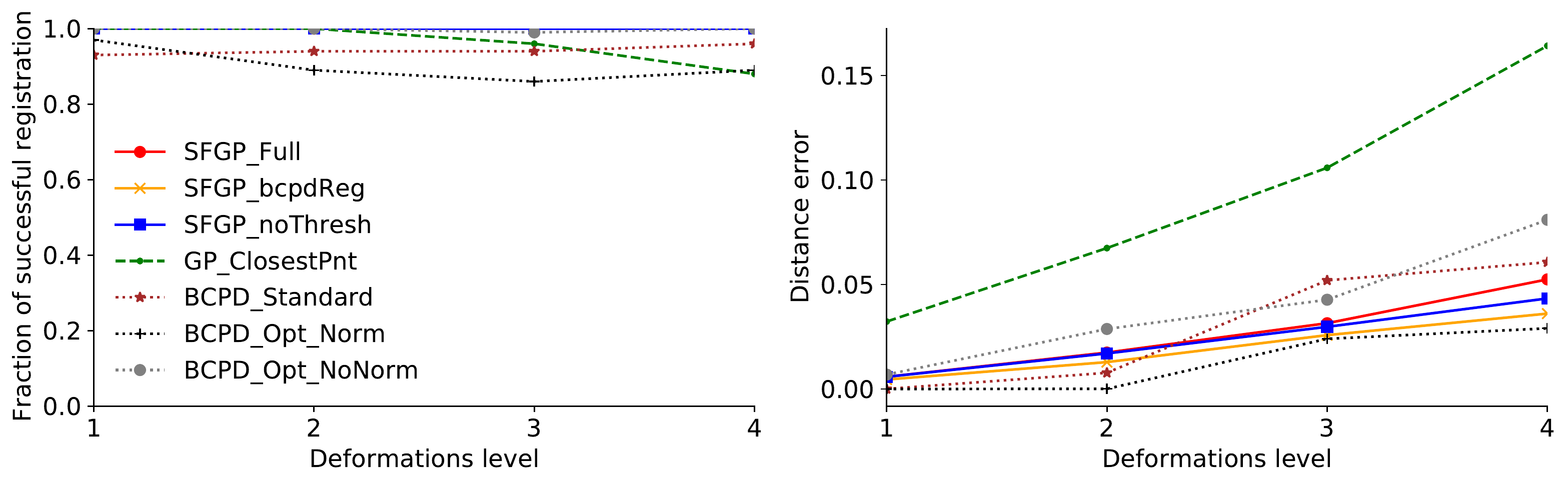}\caption{Increasing level of deformations}\label{subfig:Def}
		\end{subfigure}
		
		\begin{subfigure}[t]{\linewidth}\centering
			\includegraphics[scale=0.5]{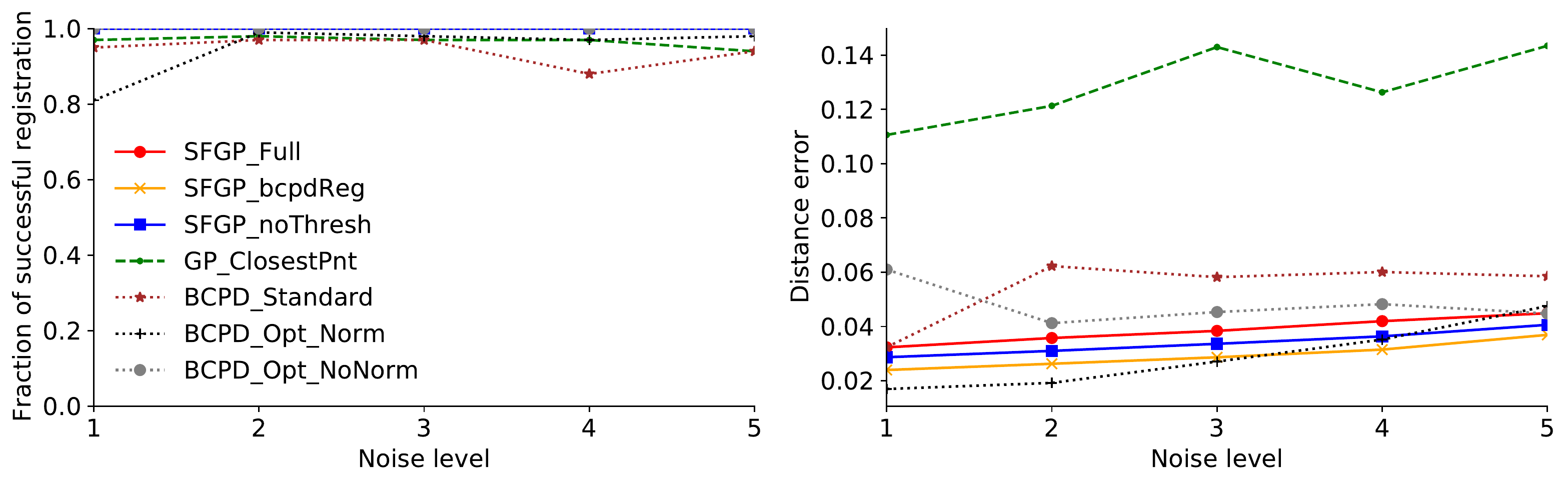}\caption{Increasing level of noise}\label{subfig:Noise}
		\end{subfigure}
		\caption{Results for the \textit{Fish Dataset} with different types of modifications. The x-axis always depicts an increase in a given data modification, while the y-axis provides the ratio of successful items registered and the average distance error. Variants of SFGP are depicted with full lines, variants of BCPD are depicted with thin dashed lines and GP with closest point with a thick dashed line. Increasing levels of missing region refer to increasing widths of the bounding box surrounding the selected reference points (width ranging from $0.1$ to $0.4$); increasing levels of deformations refer to increasing variance of the warping Gaussian radial basis function; increasing levels of noise refer to increasing variance of the additive Gaussian noise (standard deviation ranging from $0$ to $0.05$). }
		\label{fig:Fish_Tests}
	\end{figure}
	\begin{figure}[hbt!]
		\centering
		\begin{subfigure}[t]{\linewidth}\centering
			\includegraphics[scale=0.5]{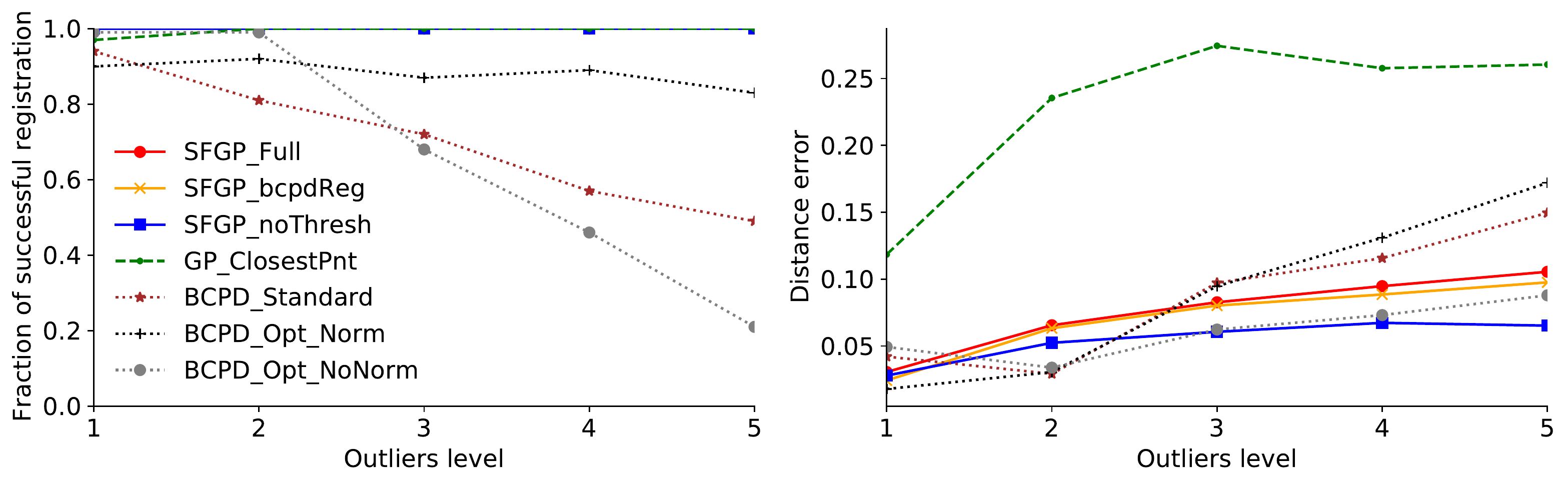}\caption{Increasing fraction of outliers, for $\omega = 0.1$}\label{subfig:Out1}
		\end{subfigure}
		
		\begin{subfigure}[t]{\linewidth}\centering
			\includegraphics[scale=0.5]{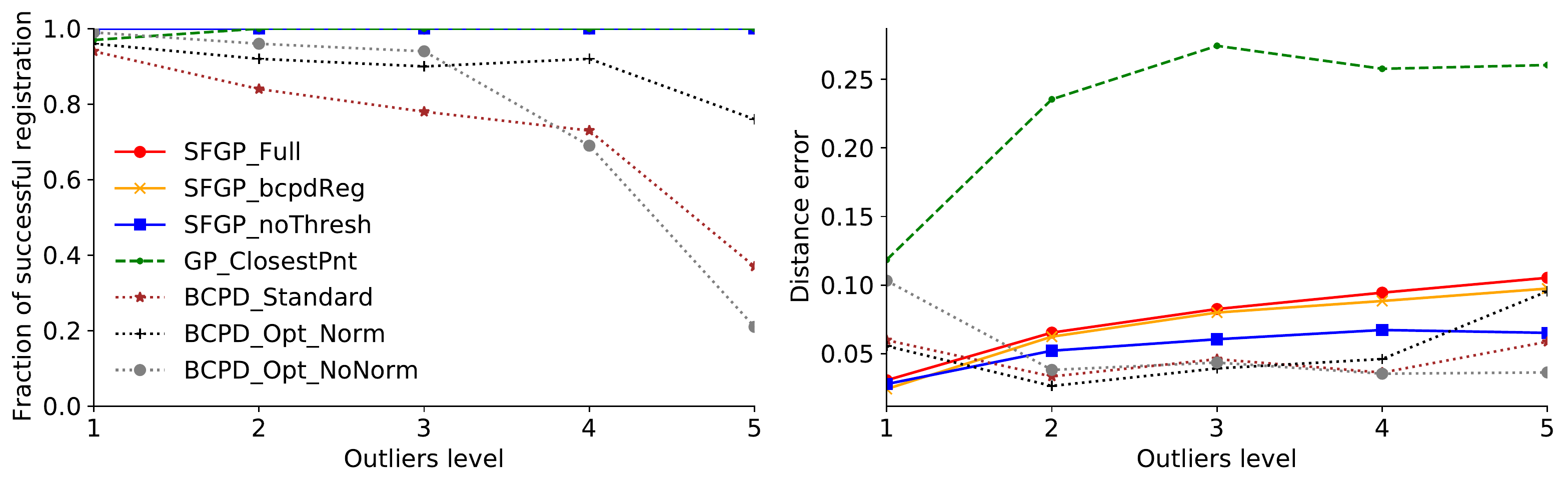}\caption{Increasing fraction of outliers, for $\omega = 0.3$}\label{subfig:Out3}
		\end{subfigure}
		
		\caption{Results for \textit{Fish Dataset} with increasing fraction of outliers. The x-axis always depicts an increase in a given data modification, while the y-axis provides the ratio of successful items registered and the average distance error. Variants of SFGP are depicted with full lines, variants of BCPD are depicted with thin dashed lines and GP with closest point with a thick dashed line. Increasing levels of outliers refer to the increasing ratio of outlier points with respect to the reference points (ranging from 0 to 2). }
		\label{fig:Fish_Tests_2}
	\end{figure}

	\begin{figure}[hbt!]\centering
		\begin{subfigure}[t]{\linewidth}\centering
			\includegraphics[scale=0.5]{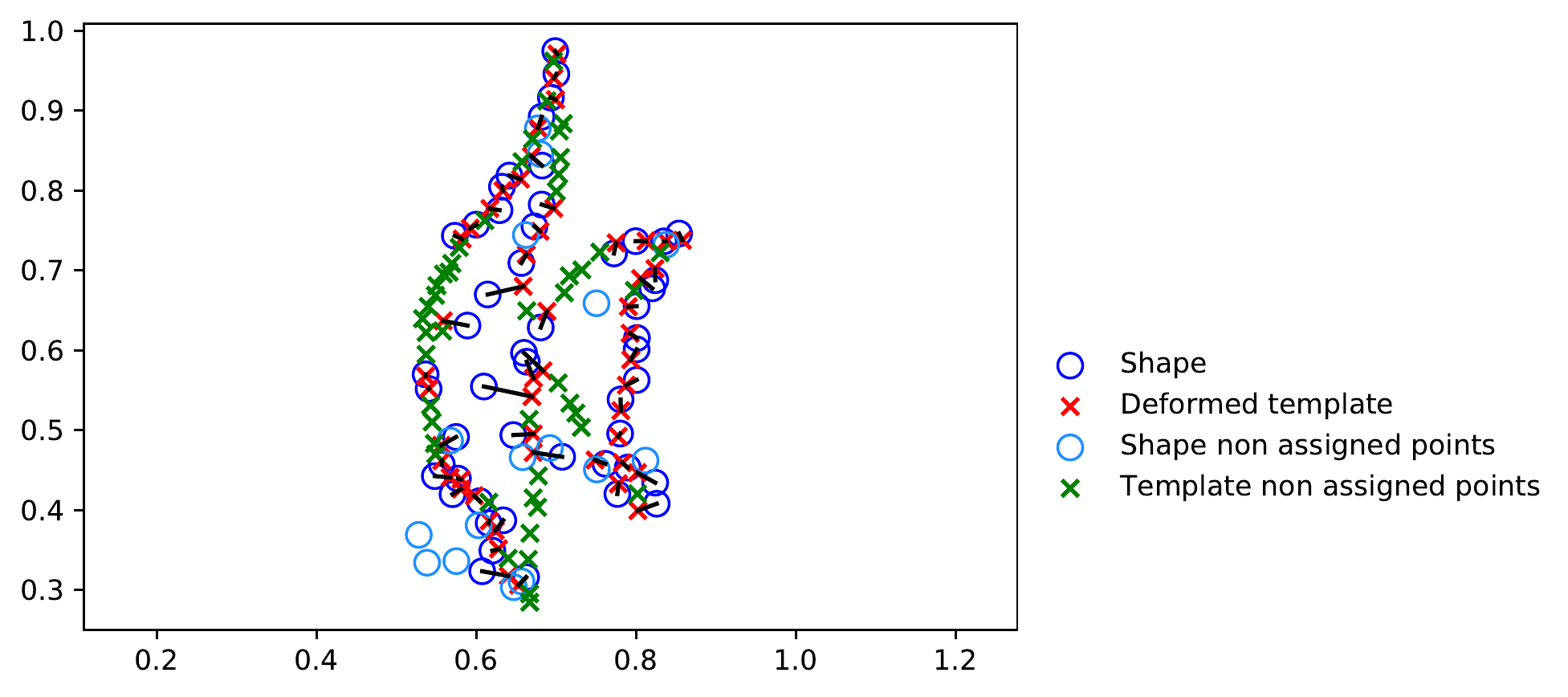}\caption{BCPD registration with low $\lambda$}\label{subfig:FishBCPDlow}
		\end{subfigure}
		
		\begin{subfigure}[t]{\linewidth}\centering
			\includegraphics[scale=0.5]{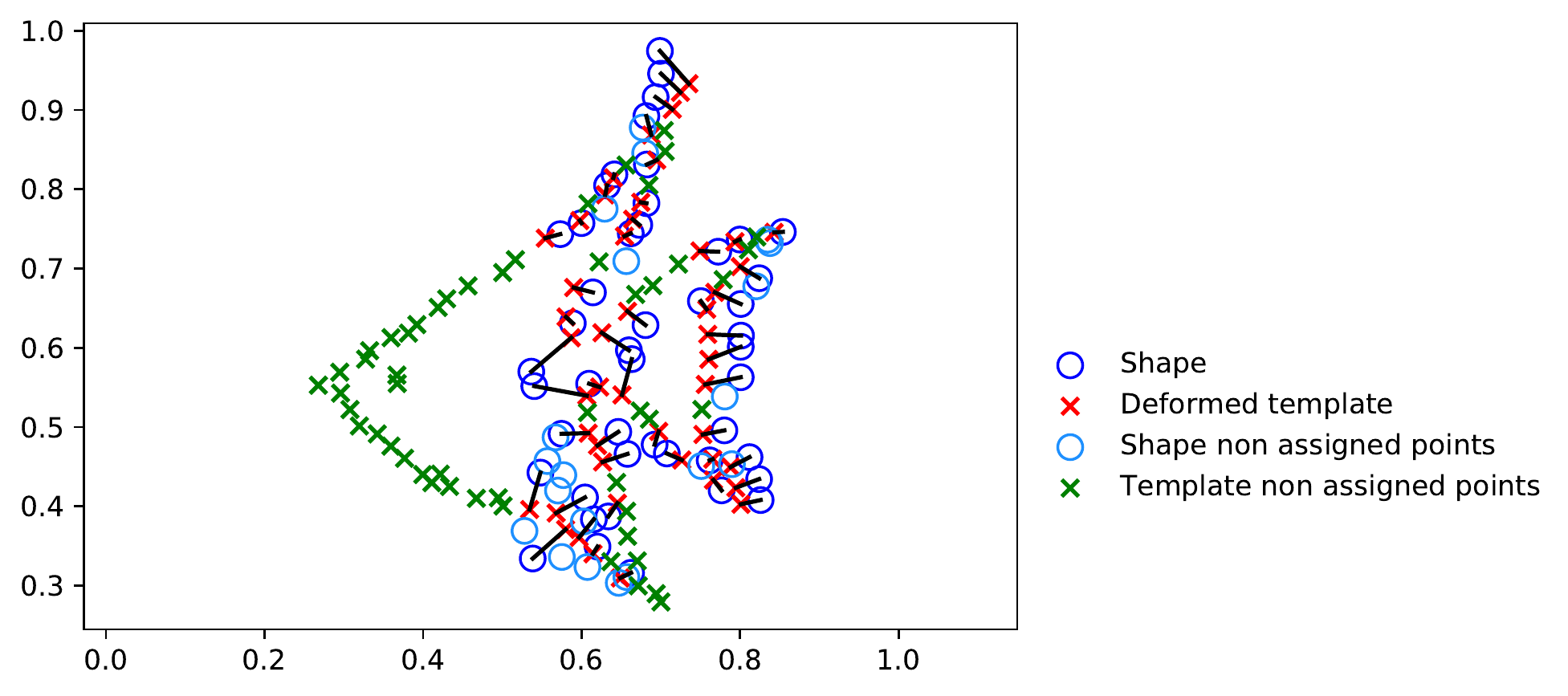}\caption{BCPD registration with high $\lambda$}\label{subfig:FishBCPDhigh}
		\end{subfigure}	
		
		\begin{subfigure}[t]{\linewidth}\centering
			\includegraphics[scale=0.5]{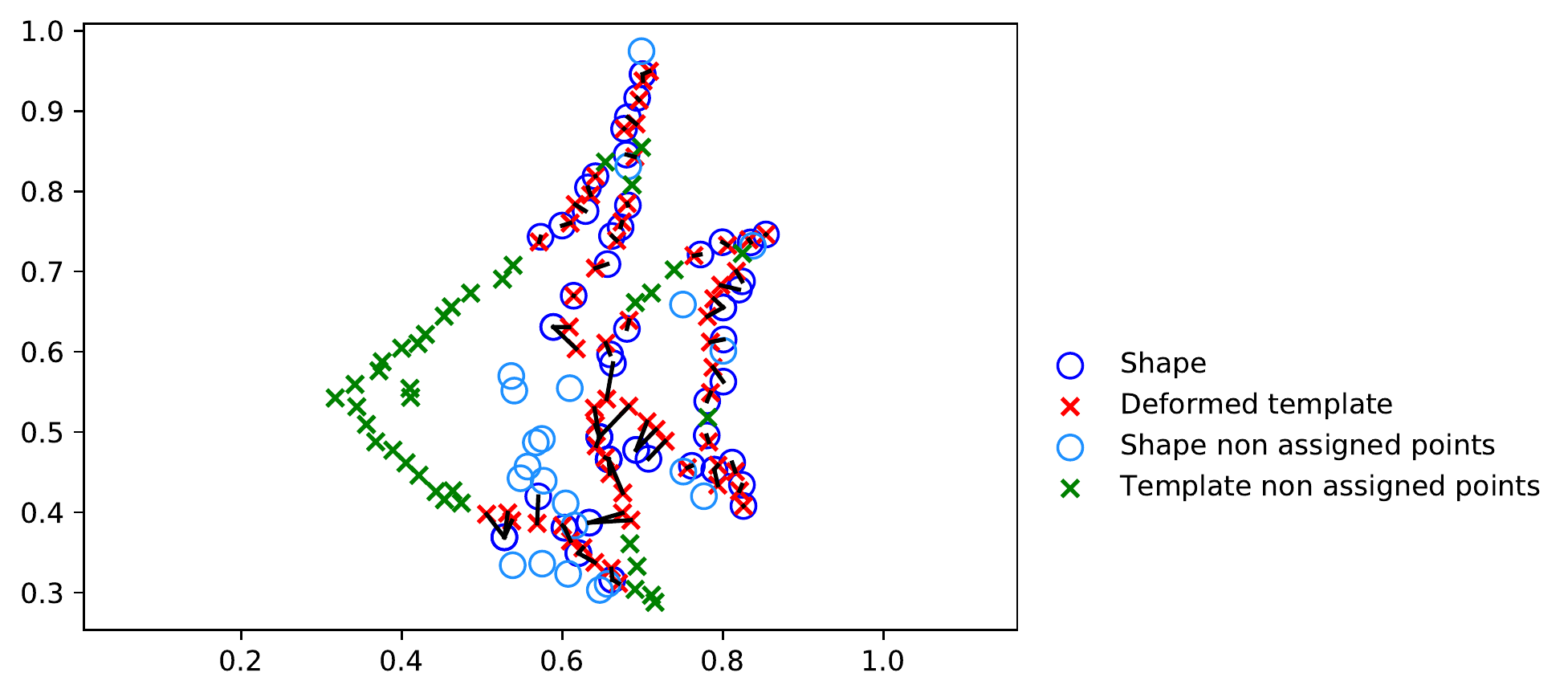}\caption{Our registration}\label{subfig:FishOurs}
		\end{subfigure}
		\caption{Example for missing data performance with \textit{Fish Dataset}. The target can be seen in blue circles, while the deformed reference after registration is represented with red and green crosses. Points with and without correspondence are identified both for the reference and target shape. The two results for BCPD are run with different values of parameter $\lambda$ responsible for controlling the expected length of the deformation --- small values of $\lambda$ allow for more deformation and vice-versa. With BCPD we can either get enough non-rigidity to fit the existing points, at cost of collapsing the missing region, or preserve the shape of this segment at the cost of a rigid transformation that can not appropriately fit the observed points. With SFGP, it is possible to allow a level of non-rigid deformations that fits the shape details, while correctly identifying the missing regions. }
		\label{fig:Fish_Example}
	\end{figure}

	\begin{figure}[hbt!]\centering
		\includegraphics[scale=0.5]{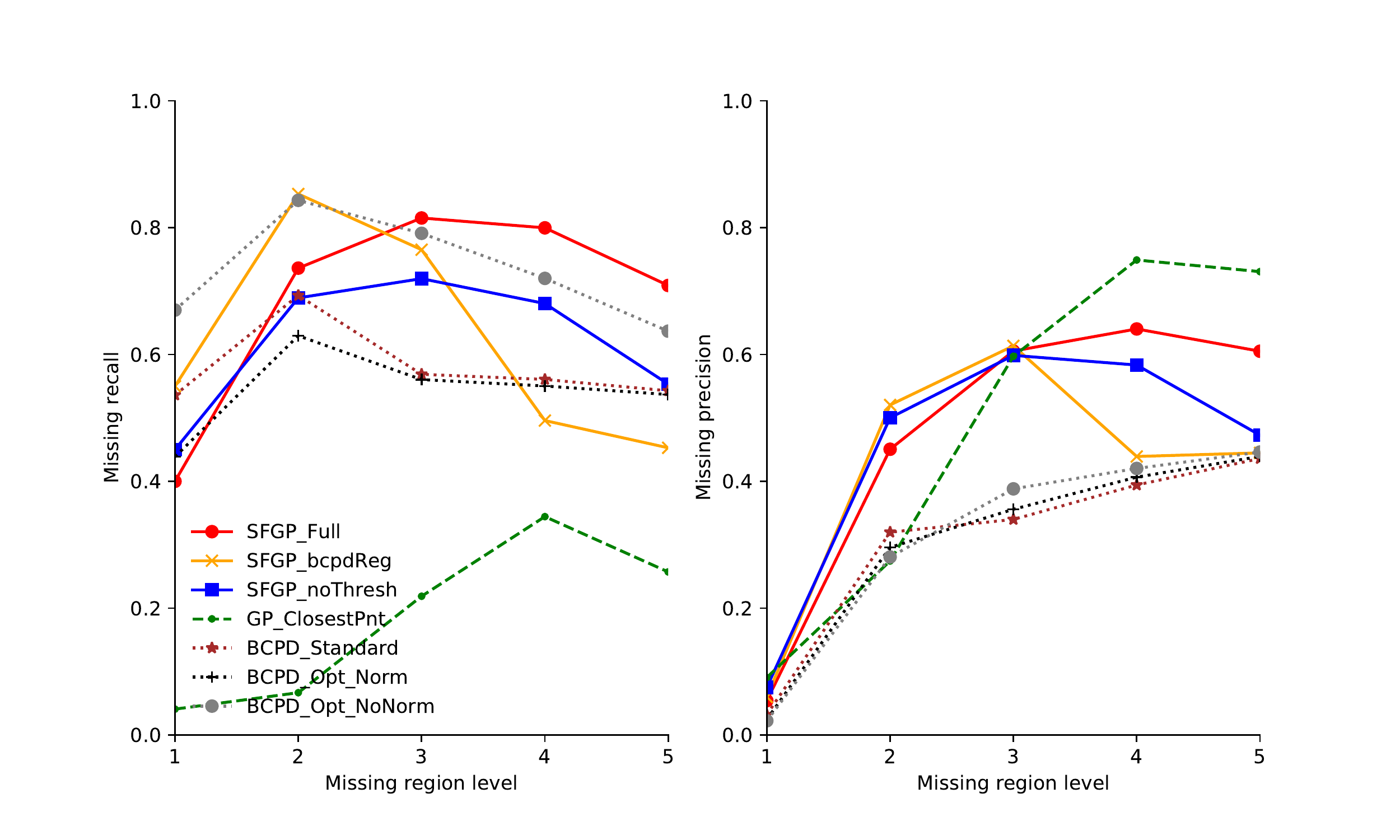}
		\caption{Recall and precision for increasing missing region.  A high recall indicates that the method is able to identify most of the missing points, while a high precision means that most of the points identified as missing are in fact missing. }
		\label{fig:Missing_details}
	\end{figure}

	\subsection{3D Ear simulated data}\label{subsec:Exp_3d_ears}
	\subsubsection{Dataset}
	
	In this section, we test our method with 3D ear data, the main goal of our work. The dataset is obtained from the \textit{Ear Dataset} in \cite{article:3DMM_Ears_Dai_Data_augmented}, with subsequent transformations to achieve more realistic shapes, mimicking the real-life challenges observed in raw scans. We denote the transformed shapes as \textit{Simulated Dataset} and an example may be found in Figure~\ref{fig:ear3d}. The transformations applied include missing data, outliers, measurement noise and a slight rotation, translation and scaling --- a more detailed description can be found in Supplementary Material (see~\ref{app:dataset_ears}).
	
	To facilitate computation, the dataset and reference were subsampled to around 3000 points. Further, we consider as possible targets only a subset of the initial 500 samples contained in the dataset, to account for the lack of variability observed between shapes (see our previous work \cite{article:UnsupervisedPipeline} for a more detailed explanation). I.e., to ensure that the reference does not closely resemble the target shapes, we select samples presenting larger deformations with respect to the reference (measured as the average of Euclidean distance between corresponding shape points). The reference was chosen as the first shape of the dataset. Usually, one tries to achieve a reference as close as possible to all shapes in the dataset (e.g. mean shape) in order to improve shape fitting. However, since we face a problem of lack of variability in our dataset, we opt for this strategy to increase the shape difference to the targets.

	\subsubsection{Discussion}
	
	Figure~\ref{fig:ears_boxplot} presents the fitting results with SFGP, BCPD \cite{article:Registration_BCPD} and ANISO \cite{article:comparison_method} for the \textit{Simulated Dataset}. For fairness, we consider the SFGP with both a Squared Exponential (SE) kernel and a sum of an SE and PCA kernel. The PCA kernel is obtained from the sample covariance of the training dataset: a subset of the original \textit{Ear Dataset} not included in the possible targets of the experiment. While the SE kernel introduces a similar prior to BCPD, the PCA kernel does so for ANISO. The distance error is presented separately for the missing and non-missing regions, for additional insight into the behaviour of each method.
	
	A first comparison is due between SFGP without the PDM kernel and BCPD. While the non-missing points present a slightly higher error for our method, this is largely compensated by the distance error occurring in the missing regions. In practice, this entails that unreasonable shapes are less likely to be originated from the fitting with SFGP, thus proving its advantage for this particular setting of extensive missing data.
	
	Secondly, the comparison between SFGP and ANISO evidences the need for a tailored method to handle outliers and missing data. The mere addition of a shape model to the probabilistic registration is not enough to overcome such obstacles (even with the anisotropic variance proposed by the authors).
	
	As expected, we note that the addition of a shape model to SFGP improves the registration results, particularly for the non-missing regions. Under this prior, our method presents a lower distance error not only for the missing regions but for the complete shape.
	
	Finally, we include an example of registration with SFGP and the closest point approach in Figure~\ref{fig:ear3d}. The limitations of the latter are clearly evidenced, as well as the capability of our method to overcome them. SFGP avoids the collapse of the posterior section when large regions of the ear are missing and adequately fits the anterior regions. Despite this, there is still room for improvement as seen on the top front region, where the deformed reference does not entirely fit the target.
	
	A remark on computational time is also due. While BCPD takes around $10$ seconds per shape with acceleration and $20$ minutes without, SFGP currently requires around $100$ minutes. However, this is an unoptimized version and does not employ any method to deal with the high dimensionality of the data, to which GPR is sensible. On one hand, the same acceleration used to obtain $P$ in BCPD, will reduce computational time. On the other, several GPR tools for large data settings already exist and could be included to reduce time complexity. Both of these should be contemplated in future work.

	\begin{figure}[hbt!]
		\centering
		\begin{subfigure}[t]{\linewidth}\centering
			\includegraphics[scale=0.5]{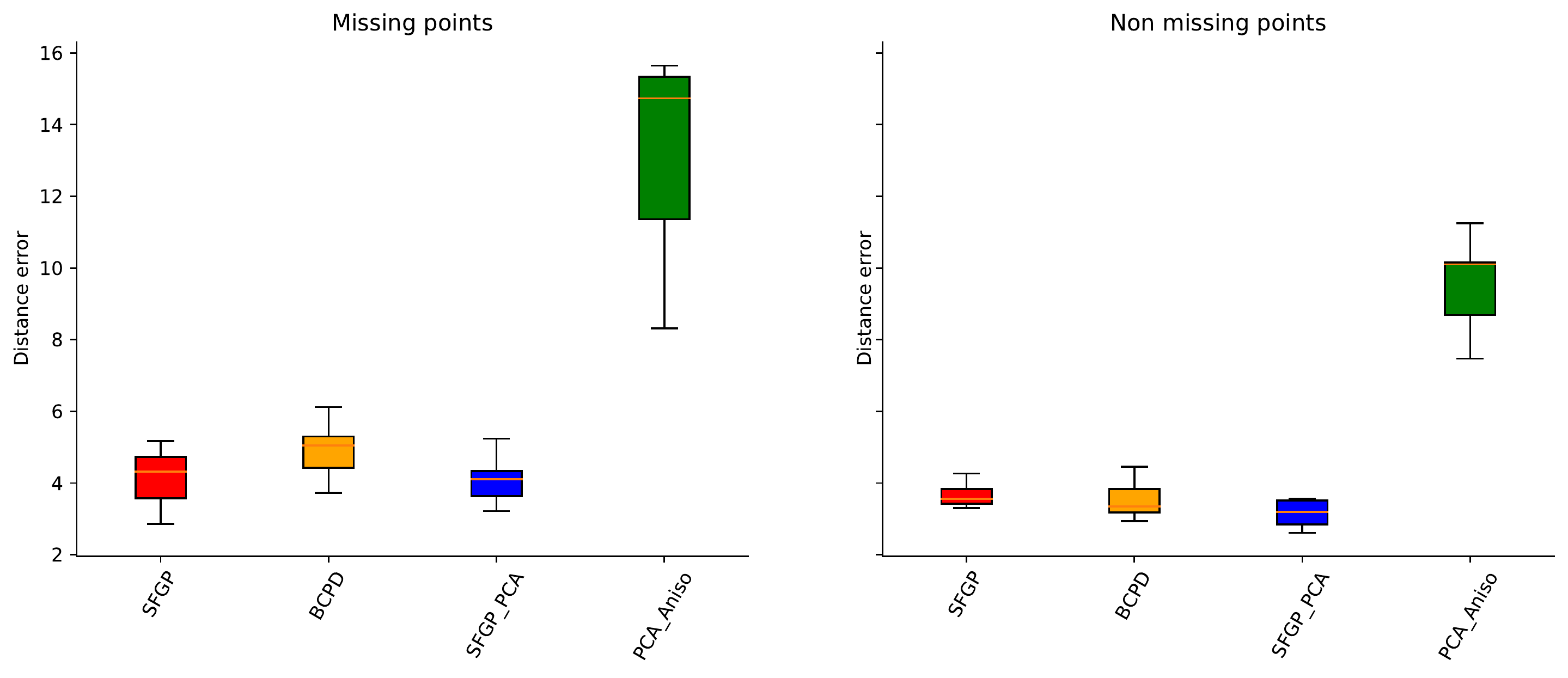}
		\end{subfigure}
		
		\begin{subfigure}[t]{\linewidth}\centering
			\includegraphics[scale=0.5]{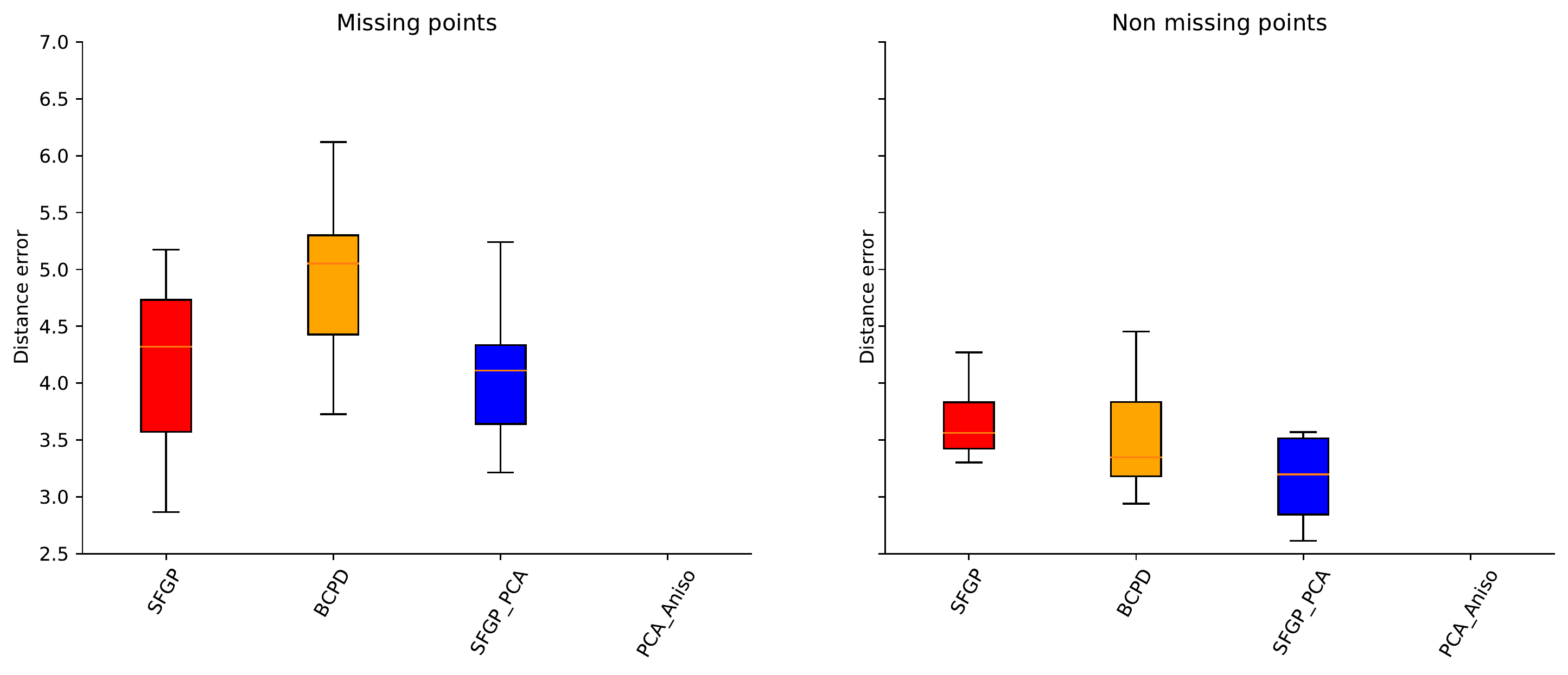}
		\end{subfigure}
		
		\caption{Results for the fitting of a reference to the \textit{Simulated Dataset} with a probabilistic registration method (\textit{BCPD}), SFGP using a Squared Exponential kernel (\textit{SFGP}), SFGP using a PCA kernel (\textit{SFGP\_PCA}) and probabilistic registration with PCA kernel and anisotropic variance (\textit{PCA\_ANISO}). The boxplot is obtained from the mean euclidean distance between the true target shapes and the deformed reference, for the entire dataset. On the left, we consider only the subset of points that are missing and, on the right, the remaining ones. The bottom row contains a zoom-in of the top plots along the $y$-axis, for better visualization of the first 3 methods. }
		\label{fig:ears_boxplot}
	\end{figure}

	\begin{figure}[hbt!]\centering
		\begin{subfigure}[t]{0.3\textwidth}\centering
			\includegraphics[clip, trim=9cm 1.8cm 9.5cm 2.2cm, scale=0.4]{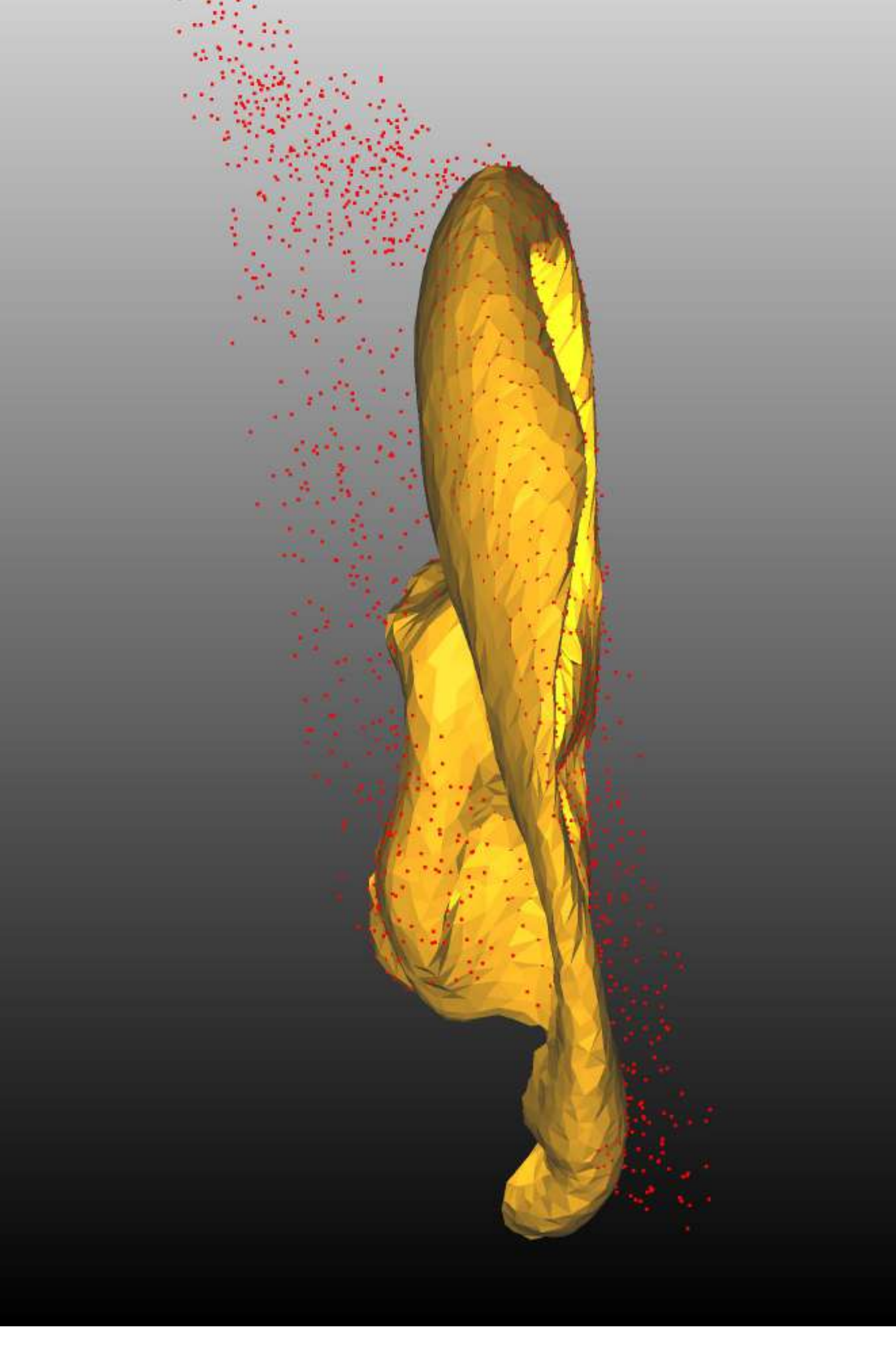}	\caption{Ground truth shape}\label{subfig:ear}
		\end{subfigure}	
		\begin{subfigure}[t]{0.3\textwidth}\centering
			\includegraphics[clip, trim=9cm 1.4cm 9.5cm 2.6cm, scale=0.4]{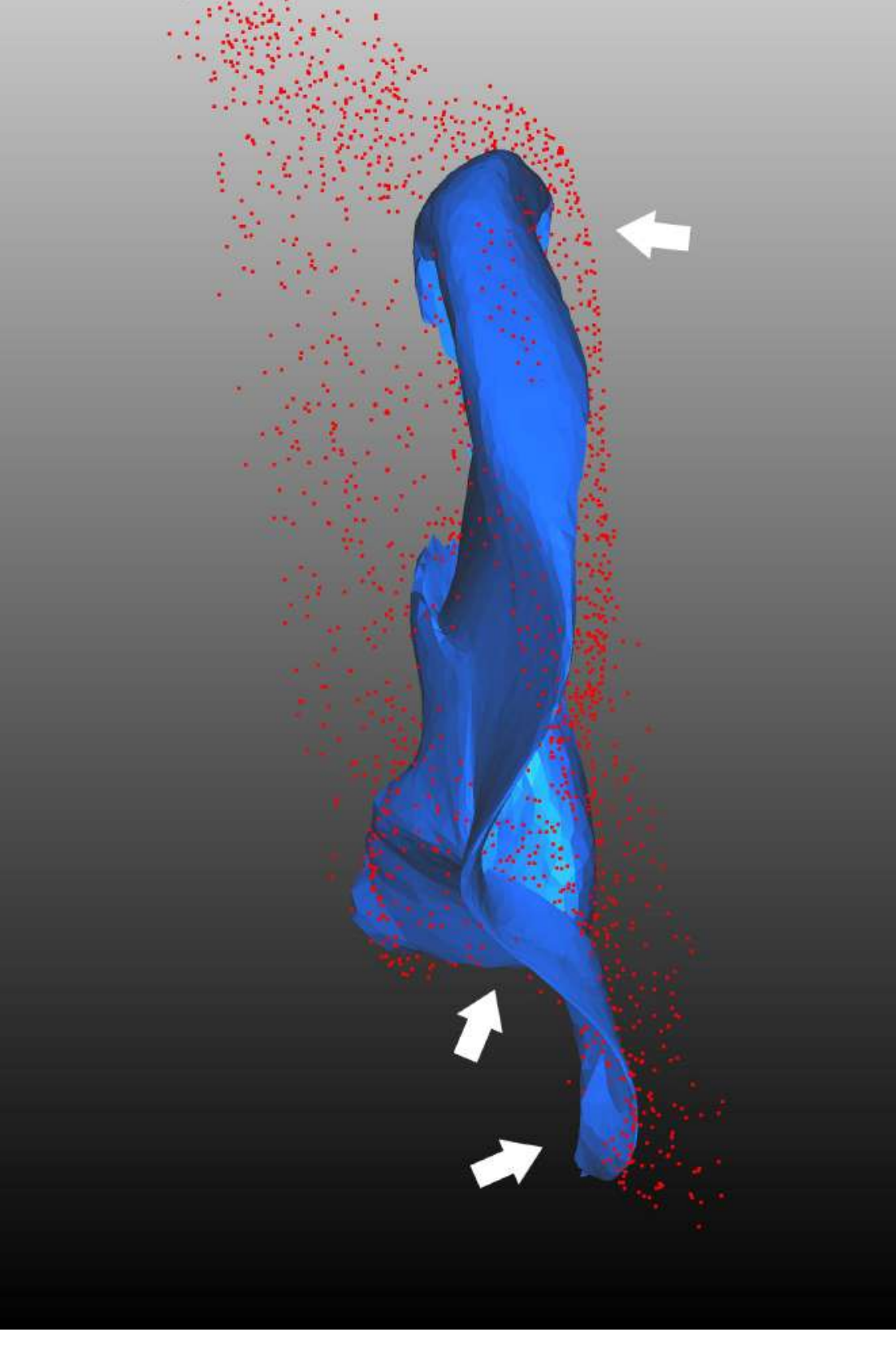}\caption{Closest Point }\label{subfig:ears}
		\end{subfigure}	
		\begin{subfigure}[t]{0.3\textwidth}\centering
			\includegraphics[clip, trim=9cm 1.5cm 9.5cm 2.5cm, scale=0.4]{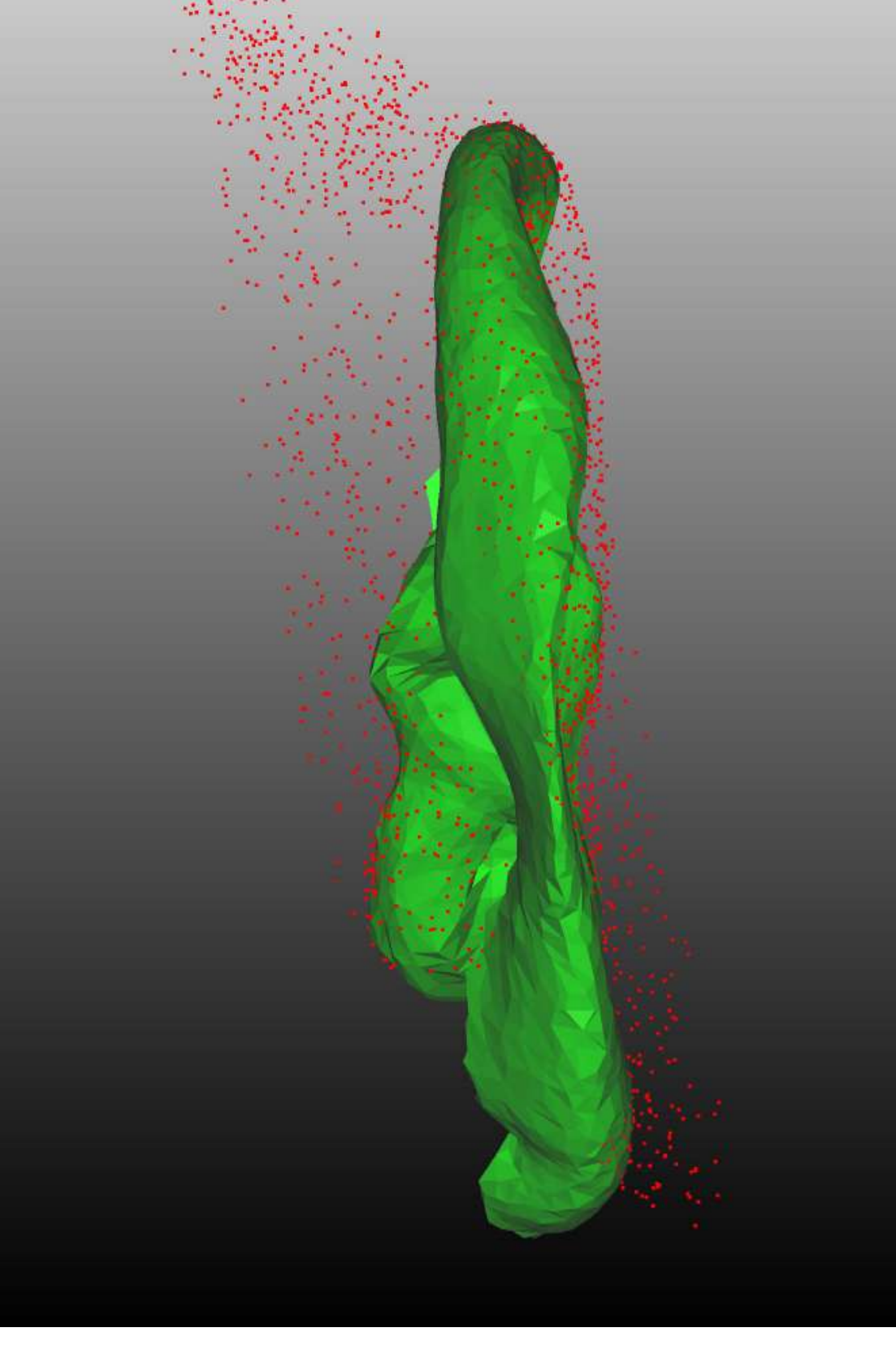}\caption{Our method, SFGP}\label{subfig:esars}
		\end{subfigure}	
		\caption{Lateral view of shape fitting for a 3D ear shape. On the left, the ground truth mesh is represented in yellow and the simulated ear as a point cloud in red (with noise, missing data and outliers). In the middle, the result obtained from the GP with closest point approach is depicted as a blue mesh. On the right, the result from the application of SFGP is represented in green. The middle approach presents three main limitations, indicated by white arrows. Near the top, it fails to completely fill the front part, which our method can only partly overcome. However, large improvements are seen in the two bottom problematic regions. The large missing area (middle arrow) and the presence of data only for the front part of the ear (bottom arrow), both cause the reference to collapse for the middle approach. SFGP is able to overcome these challenges leading to an increased resemblance to the original shape.}
		\label{fig:ear3d}
	\end{figure}

	\section{Concluding remarks and future work}
	\label{sec:conclusions}
	
	We developed a method that bridges the gap between the Gaussian Process framework used in 3D Morphable Models and the probabilistic registration methods, by formulating the shape fitting problem in a GPR multi-annotator setting. This allows us to benefit from advantages on both sides and obtain a method particularly suited for shape fitting in the presence of extensive missing data --- a useful tool for challenging shapes such as the human ear.
	
	Naturally, even if the missing points are correctly identified, the shape prediction in those regions will be as good as the prior model. Therefore, it is beneficial to have a more complex and accurate model, able to express more knowledge regarding the particular shape. As stated, the GP framework offers a very suitable setting, with kernels expressing intuitive properties of the shapes. So, defining a more appropriate kernel is the logical next step. It would also be pertinent to study how the parallel with the probabilistic registration holds when we introduce the missing point set and the threshold, to have a more theoretical insight into the properties of our method.

	\appendix

	\section*{Acknowledgments}
	Computational resources provided by INDACO Platform, which is a project of High Performance Computing at the University of Milan \url{http://www.unimi.it}.
	
	\bibliographystyle{siam}
	\bibliography{references}
	
\end{document}